\documentclass[final,5p,times,twocolumn]{elsarticle}

\journal{Computer Standards \& Interfaces}

\usepackage{amsmath,amssymb,amsthm,mathtools}
\usepackage[T1]{fontenc}
\usepackage[utf8]{inputenc}
\usepackage{newtxmath}
\usepackage{graphicx}
\usepackage{booktabs}
\usepackage{tabularx}
\usepackage{array}
\usepackage{tikz}
\usetikzlibrary{arrows.meta, positioning}
\usepackage[ruled,linesnumbered]{algorithm2e}
\usepackage{listings}
\usepackage{xcolor}
\usepackage{hyperref}
\usepackage[capitalize,nameinlink]{cleveref}

\definecolor{codegreen}{rgb}{0,0.6,0}
\definecolor{codegray}{rgb}{0.5,0.5,0.5}
\definecolor{codepurple}{rgb}{0.58,0,0.82}
\definecolor{codeblue}{rgb}{0.13,0.29,0.53}
\definecolor{backcolour}{rgb}{0.97,0.98,0.99}

\DeclareMathOperator{\argTopK}{argTopK}
\DeclareMathOperator{\softmax}{softmax}

\newtheorem{theorem}{Theorem}
\newtheorem{corollary}{Corollary}
\newtheorem{lemma}{Lemma}
\newtheorem{proposition}{Proposition}
\theoremstyle{definition}
\newtheorem{definition}{Definition}
\theoremstyle{remark}
\newtheorem{remark}{Remark}

\hypersetup{
    colorlinks=true,
    linkcolor=blue!70!black,
    citecolor=red!70!black,
    urlcolor=blue!70!black
}
\newcommand{\method}{Page-EntroKV}
\newcommand{\HQ}{H_{Q}}
\newcommand{\HKV}{H_{\mathrm{KV}}}
\newcommand{\Dmodel}{D_{\mathrm{model}}}
\newcommand{\Dhead}{D_{\mathrm{head}}}
\newcommand{\Bbatch}{B_{\mathrm{batch}}}
\newcommand{\Nblk}{N_{\mathrm{blocks}}}
\newcommand{\Attn}[2]{\mathbf{A}_{#1,#2}}
\newcommand{\AttnPool}[2]{\mathbf{A}^{\mathrm{pooled}}_{#1,#2}}
\newcommand{\AttnMean}[2]{\mathbf{A}^{\mathrm{mean}}_{#1,#2}}
\newcommand{\Renyi}[2]{\mathcal{H}_{#1}\!\left(#2\right)}
\newcommand{\Collision}[1]{\mathcal{H}_{2}\!\left(#1\right)}
\newcommand{\CollisionSink}[1]{\mathcal{H}^{\mathrm{sink}}_{2}\!\left(#1\right)}
\newcommand{\UOR}{\mathrm{UOR}}
\newcommand{\Group}[1]{\mathrm{Group}(#1)}
\newcommand{\PageSet}{\mathcal{P}}
\newcommand{\RetSet}[2]{\mathcal{K}_{#1,#2}}
\newcommand{\TokSet}[1]{\mathcal{K}_{l,#1}}
\newcommand{\PageImp}[3]{I_{#1,#2}(P_{#3})}
\newcommand{\Ssink}{S_{\mathrm{sink}}}
\newcommand{\Hsink}{\mathcal{H}^{\mathrm{sink}}_{2}}
\newcommand{\GB}{\,\text{GB}}
\newcommand{\TBs}{\,\text{TB/s}}
\newcommand{\TFLOPs}{\,\text{TFLOP/s}}

\tikzset{
  flow/.style={draw, rounded corners=3pt, fill=blue!6,
               text width=0.86\textwidth, inner sep=6pt,
               align=left, font=\small},
  qbox/.style={draw, fill=blue!8, minimum width=22mm,
               minimum height=9mm, align=center,
               rounded corners=2pt, font=\small},
  pbox/.style={draw, fill=orange!10, text width=0.82\textwidth,
               inner sep=6pt, align=center, rounded corners=2pt,
               font=\small},
  arr/.style={-{Latex[length=2mm, width=1.4mm]}},
}

\lstdefinestyle{mystyle}{
  backgroundcolor=\color{backcolour},
  commentstyle=\color{codegreen},
  keywordstyle=\color{codeblue}\bfseries,
  numberstyle=\tiny\color{codegray},
  stringstyle=\color{codepurple},
  basicstyle=\ttfamily\scriptsize,
  breaklines=true, keepspaces=true, numbers=left, numbersep=5pt,
  showspaces=false, showstringspaces=false, showtabs=false,
  tabsize=2, frame=single, rulecolor=\color{gray!35},
  xleftmargin=12pt, framexleftmargin=12pt, captionpos=b
}
\begin{document}

\begin{frontmatter}

\title{Page-EntroKV: Hardware-Aligned, Entropy-Weighted KV-Cache
  Eviction under Grouped-Query Attention\\[1.6em]
  {\normalfont\normalsize
    \begin{tabular}{@{}c@{}}
      \textbf{Inbasekaran S.}\textsuperscript{*}\\[3pt]
      Department of Computer Science and Engineering\\
      SRM Institute of Science and Technology, India\\[3pt]
      \texttt{iambruce0012@gmail.com}
    \end{tabular}}%
\\[0.6em]
  {\normalfont\footnotesize
    \textsuperscript{*}\,Corresponding author.}}

\author{}

\begin{abstract}
Long-context inference in autoregressive language models is limited by
the key--value (KV) cache. The cache grows with context length, and
every decoded token streams the full cache through memory. Most dynamic
eviction methods score token importance per query head and choose
tokens independently. This fits poorly with grouped-query attention
(GQA), where several query heads share one physical KV buffer: if the
heads choose different tokens, the serving engine must keep the union
of their selections, inflating the cache by up to the group ratio $r$.
A simple arithmetic average does not fix this---it dilutes the
specialized retrieval heads that carry factual recall.

We introduce \textbf{Page-EntroKV}, a formal framework for KV-cache
eviction that works at the level GQA serving actually uses. Heads in
each physical group are pooled with weights derived from sink-isolated
collision (R\'enyi-2) entropy---one inner product per head, computed
once at prefill with no calibration---so sink heads cannot be mistaken
for retrieval heads. The pooled scores are mapped onto PagedAttention
page frames, and eviction runs at the hardware tuple $(l, g, P_b)$. We
formalize the union overhead ratio and intra-group disagreement, prove
an exact identity linking them for two-head groups, and give two-sided
bounds for general ratios. We prove that group-level pooling preserves
the eviction budget exactly, derive a finite-context needle-retention
guarantee that arithmetic-mean pooling provably violates, characterize
the temperature limits of the entropy-weighted family, and give exact
per-layer cross-group page accounting.

On a pilot architecture (Qwen2.5-1.5B-Instruct, $r{=}6$, 28 layers),
head-independent replay inflates the cache by up to $5.1\times$ at a
 $2\%$ budget---with the theoretical certificate firing in the large
majority of measured cells---while \method{} holds the union overhead
at exactly $1.000$ across 1{,}008 pooled measurements. Across a
1{,}008-observation head census, sink isolation inverts the
sink-to-context weight ordering from a $2.10\times$ unearned sink
advantage to a $0.66\times$ deficit. In pilot-scale needle retrieval,
entropy-weighted pooling retains the needle where arithmetic-mean
pooling collapses ($100\%$ vs.\ $0\%$ under the pooled-selection
criterion at a $20\%$ budget); strict all-layer retention is reported
as a limitation. Retained cardinality is exact for every page size,
preliminary runs show representative QA and code tasks remain solvable
at $20\%$ retention, and the mechanism ablations are certified over all
56 of 56 layer--group pairs with zero skipped layers. Larger-model,
long-context, and serving phases are pre-registered and in progress.
\end{abstract}


\end{frontmatter}

\section{Introduction}
\label{sec:introduction}

Long-context large language models (LLMs) are now deployed for
multi-turn agentic workflows, repository-scale code synthesis, and
multi-document legal
discovery~\cite{bai2023longbench,hsieh2024ruler,meta2024llama3}.
Context windows have grown to match, but long-context use remains
fragile: reliability degrades with input position~\cite{liu2023lost}.
At context lengths of 32k to 128k$^{+}$ tokens, decoding hits a hardware
boundary: every decoded token streams the entire key--value (KV) cache
from GPU high-bandwidth memory (HBM) into on-chip SRAM, so the decode
phase is strictly memory-bound and its cost grows linearly with context
length~\cite{pope2022scaling,kwon2023vllm}. At practical batch sizes the
KV cache alone can occupy most of the accelerator's memory, capping
concurrency and decoding throughput.

Production architectures respond by sharing KV buffers across query
heads. Grouped-query attention (GQA)~\cite{ainslie2023gqa}---adopted by
Llama-3.1, Qwen-2.5, and
Mistral~\cite{meta2024llama3,qwen2024qwen25,jiang2023mistral}---partitions
 $\HQ$ logical query heads into $\HKV$ physical KV groups, shrinking the
cache by the ratio $r=\HQ/\HKV$. In serving runtimes such as
vLLM~\cite{kwon2023vllm}, these groups are materialized as
block-partitioned tensors of shape $[\Nblk,\ \HKV,\ B,\ \Dhead]$:
the sharing is the memory layout itself, not an implementation detail.

This physical layout breaks a core assumption behind prevailing eviction
heuristics. Policies such as H2O~\cite{zhang2023h2o},
SnapKV~\cite{li2024snapkv}, and Scissorhands~\cite{liu2023scissorhands}
compute token importance per query head and select top-$k$ sets
independently. Sibling heads in a GQA group share one physical buffer,
so the engine must store the \emph{union} of their selections---up to
 $r\times$ the intended budget. We formalize this penalty as the
\emph{union overhead ratio} (UOR) and connect it exactly to measurable
intra-group disagreement (\cref{prop:uorjaccard}). Arithmetic mean
pooling fails from the other direction: it dilutes specialized retrieval
heads~\cite{wu2024retrievalheads}, averaging a needle token attended to
with mass $1-\epsilon$ by one head down to $(1-\epsilon)/r$, below
eviction thresholds.

\method{} follows three design principles, each chosen against a
specific alternative. Pooling at the physical-group level rather than
the head level preserves GQA's memory sharing and keeps the budget
exact; head-level selection cannot do both, because the engine stores
whatever union the heads produce. Discrete attention
entropy---one inner product per head---separates sharp retrieval heads
from diffuse background heads, so weighting heads by inverse entropy
amplifies exactly the signals that mean pooling destroys; evaluating
that entropy on sink-isolated distributions keeps attention-sink heads
from being mistaken for retrieval heads. Projecting importance onto
discrete PagedAttention page frames \emph{before} eviction, rather than
selecting tokens first, makes the compression decision coincide with the
hardware allocation tuple $(l, g, P_b)$ and eliminates internal
fragmentation and cache-unrolling overhead.
\Cref{fig:pipeline_overview} shows the resulting architecture.

The contributions of this paper are formal, algorithmic,
implementational, and pilot-validated:
\begin{itemize}
  \item \textbf{Architectural formulation.} We formalize the GQA
        structural constraint in dynamic KV-cache eviction through two
        metrics---the union overhead ratio and intra-group
        disagreement---linked by an exact identity for two-head groups
        and sandwiched two-sidedly at every ratio
        (\cref{sec:background}).
  \item \textbf{Algorithmic design.} We design entropy-weighted group
        pooling around discrete sink-isolated R\'enyi-2 attention
        entropy, together with a page-frame mapping that evicts exactly
        at the $(l, g, P_b)$ tuple (\cref{sec:methodology}).
  \item \textbf{Theoretical guarantees.} We prove strict budget
        preservation ($\UOR \equiv 1$), a finite-context
        needle-retention bound under head specialization with sink
        isolation, the temperature degeneracies of the pooling family,
        and page-level union accounting (\cref{sec:theory}).
  \item \textbf{Reference implementation and pilot validation.} We
        provide a complete PyTorch reference implementation, an
        instrumented harness, and a pilot validation on
        Qwen2.5-1.5B-Instruct confirming the structural core on real
        text---with every deviation, exclusion, and pending item
        disclosed (\cref{sec:protocol}).
\end{itemize}

This is a formulation-and-theory paper with a pilot-scale empirical
tier. The architectural guarantees hold by construction; the pilot of
\cref{sec:protocol} validates the structural core on real text at small
scale; the full pre-registered sweep---long context, larger models, and
serving---is in progress and disclosed item by item in
\cref{subsec:pilot-pending}.

The remainder of this paper is organized as follows.
\Cref{sec:related} reviews related work. \Cref{sec:background} formalizes
the problem and its structural metrics. \Cref{sec:methodology} presents
the \method{} design. \Cref{sec:theory} develops the theoretical
analysis. \Cref{sec:protocol} reports the pilot validation and the
remaining pre-registered protocol. \Cref{sec:discussion} delimits scope
and limitations, and \cref{sec:conclusion} concludes.
\Cref{app:notation} summarizes notation, \cref{app:verification}
tabulates the guarantee-to-evidence mapping,
\cref{app:derivations} collects the full step-by-step derivations, and
\cref{app:implementation} gives the reference implementation.

\begin{figure*}[t]
\centering
\begin{tikzpicture}[font=\small, node distance=4mm]
  \node[flow] (input)
    {\textbf{Input context} --- long-form text
     ($T \in [32\text{k},\,128\text{k}^{+}]$ tokens).};
  \node[flow, below=of input] (attn)
    {\textbf{Transformer layer $l$ prefill attention.} Produces per-head
     distributions $\{\Attn{l}{h}\}_{h=0}^{\HQ-1}$.};
  \node[flow, below=of attn] (pool)
    {\textbf{Sink-isolated entropy estimation \& entropy-weighted group
     pooling.} Quantifies head dispersion via sink-isolated
     $\CollisionSink{\Attn{l}{h}}$ over the observation window; Boltzmann
     weights $w_{l,h} \propto e^{-\CollisionSink{\Attn{l}{h}}/\tau_g}$      pool heads at the physical GQA group, eliminating union overhead
     ($\UOR \to 1$) and preventing needle dilution.};
  \node[flow, below=of pool] (page)
    {\textbf{Block-max page mapping.} Importance scores are aggregated
     onto physical pages $P_b$, and the top-$K_l$ pages are selected per
     group at the hardware tuple $(l, g, P_b)$.};
  \node[flow, below=of page] (vllm)
    {\textbf{vLLM PagedAttention engine.} Physical layout
     $[\Nblk,\ \HKV,\ B,\ \Dhead]$; zero internal fragmentation,
     group-level $\UOR \equiv 1$, strict budget of $k$ tokens per
     group.};
  \foreach \a/\b in {input/attn, attn/pool, pool/page, page/vllm}
    \draw[arr] (\a.south) -- (\b.north);
\end{tikzpicture}
\caption{\method{} pipeline overview. Sink-isolated entropy-weighted
  pooling at the physical GQA group level, followed by block-max
  mapping onto PagedAttention page frames, aligns the eviction decision
  with the engine's physical memory layout.}
\label{fig:pipeline_overview}
\end{figure*}
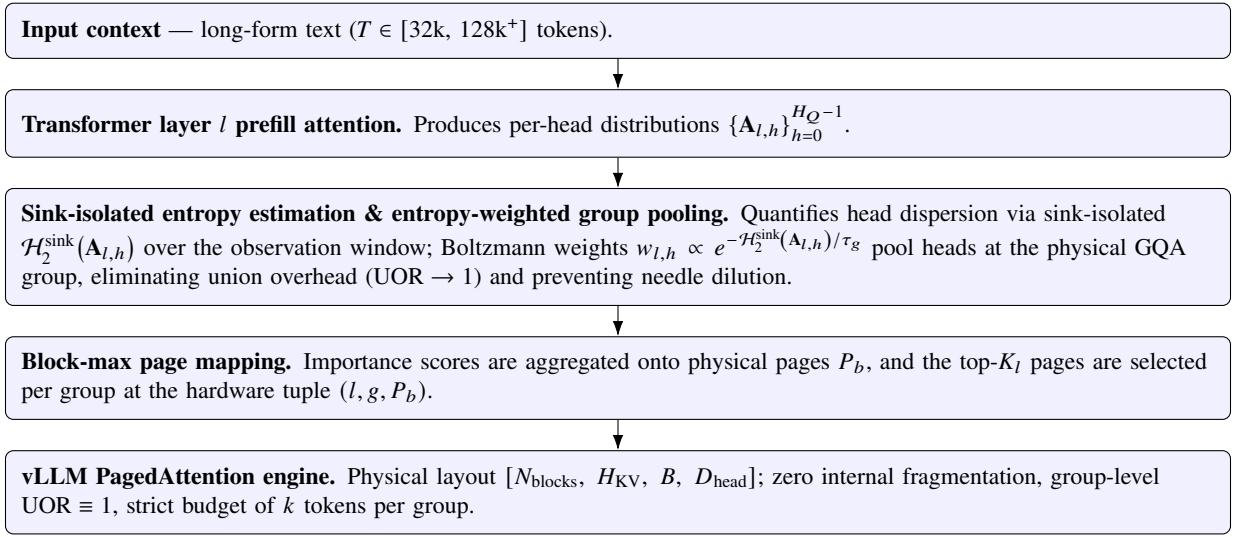

\section{Related work}
\label{sec:related}

\subsection{Attention efficiency and architectural evolution}
Multi-head attention (MHA)~\cite{vaswani2017attention} pairs every query
head with private key--value buffers. Multi-query attention
(MQA)~\cite{shazeer2019fast} collapses all heads onto a single shared
buffer. GQA~\cite{ainslie2023gqa} interpolates between the two, cutting
KV memory by the group ratio $r$ at negligible quality cost, and is now
standard in production
models~\cite{meta2024llama3,qwen2024qwen25,jiang2023mistral}. A separate
line compresses KV states into low-rank latent vectors, as in the
multi-head latent attention (MLA) of
DeepSeek-V2~\cite{deepseek2024deepseekv2}. These works fix the
\emph{static} allocation of KV memory. \method{} addresses the
complementary \emph{dynamic} problem: deciding at inference time which
allocated entries to keep.

\subsection{Inference serving engines}
IO-aware kernel design minimizes HBM traffic.
FlashAttention~\cite{dao2022flashattention} tiles attention computation
on-chip; FlashAttention-2 adds flash-decoding, improving parallelism for
long sequences and decode workloads~\cite{dao2023flashattention2}. On
the systems side, vLLM's PagedAttention~\cite{kwon2023vllm} partitions
the KV cache into fixed-size blocks managed by a block table and reaches
near-zero fragmentation at high concurrency. Orca raises utilization
further through iteration-level scheduling and continuous
batching~\cite{yu2022orca}. Roofline analyses~\cite{pope2022scaling}
establish the memory-bound character of decoding. These engines assume
the cache is fully materialized; they define the physical layout to
which any compression policy must conform. \method{} treats the page
table as a first-class citizen and evicts at exactly the granularity the
engine allocates.

\subsection{Dynamic KV-cache eviction}
H2O~\cite{zhang2023h2o} retains tokens with the largest accumulated
attention scores. StreamingLLM~\cite{xiao2024streamingllm} preserves
attention sinks and a sliding window.
Scissorhands~\cite{liu2023scissorhands} exploits the persistence of
token importance. SnapKV~\cite{li2024snapkv} compresses the prefill cache
using observation windows. PyramidKV~\cite{cai2024pyramidkv} allocates
layer-wise budgets following a pyramidal information funnel.
Quest~\cite{tang2024quest} performs query-aware sparse selection.
FastGen~\cite{ge2024fastgen} selects per-head compression policies from
observed attention patterns, adapting the operator while staying at
head-level granularity. Ada-KV~\cite{feng2024adakv} works on the budget
side of the same GQA structure: it reallocates the per-head eviction
budgets within a group adaptively, so the union of head-level selections
respects the group budget instead of a uniform per-head quota $k$.
Ada-KV decides \emph{how much} each head may keep; it does not change
\emph{which} entries survive---selection remains per-head on raw
attention scores, with no selection-side guarantee comparable to the
needle-retention bound of \cref{thm:needle_preservation}.

All of these methods score importance at the (layer, head) or (layer,
token) granularity. Under shared buffers, independent head-level
selections inflate memory by up to $r\times$ unless a budget allocator
compensates---which is the compensation Ada-KV supplies---and per-head
scores cannot be materialized without unrolling the cache. \method{}
attacks the inflation at its source: pooling before selection makes the
union semantics of GQA harmless by construction, and eviction runs at
page granularity.

A concurrent entropy-based method, EntroKV~\cite{entrokv2025}, applies
attention entropy at a different decision point. There, Shannon entropy
allocates per-layer compression budgets---\emph{how much} each layer
keeps---as a plug-and-play front-end, with offline calibration of a
mean-entropy statistic $\bar{\mathcal{H}}$. \method{} applies
sink-isolated R\'enyi-2 entropy on the selection side---\emph{which}
pages survive within each physical group---requires no calibration, and
carries the guarantees proved here: group-level $\UOR\equiv 1$ (\cref{thm:memory_guarantee}) and finite-context needle retention
(\cref{cor:finite_context_bound}). The two mechanisms sit at different
points in the pipeline (budget allocation versus within-group
selection); no combination of the two is studied here.
\Cref{tab:qualitative} positions \method{} against these families, and
against channel-dimension and routing alternatives, along the axes that
the GQA layout makes salient.

\begin{table*}[t]
\centering\small
\caption{Qualitative structural comparison of KV-cache eviction
  approaches under physically shared GQA buffers.
  \checkmark\ = supported; $\sim$ = partially; $\times$ = not supported;
  n/a\ = not applicable: the method is not a token-eviction policy
  (channel-dimension pruning or quantization) or is a context-routing
  architecture.}
\label{tab:qualitative}
\begin{tabular}{@{}lccccc@{}}
\toprule
Method & Importance & GQA-union & Eviction & Strict budget & Calibration \\
 & granularity & aware? & granularity & guarantee? & -free? \\
\midrule
H2O~\cite{zhang2023h2o}           & $(l,h,t)$ & $\times$ & token & $\times$ (union $\le r\times$) & \checkmark \\
StreamingLLM~\cite{xiao2024streamingllm} & position & \checkmark\ (uniform) & token & \checkmark\ (static) & \checkmark \\
Scissorhands~\cite{liu2023scissorhands} & $(l,h,t)$ & $\times$ & token & $\times$ & \checkmark \\
SnapKV~\cite{li2024snapkv}        & $(l,h)$ & $\times$ & token & $\times$ & \checkmark \\
FastGen~\cite{ge2024fastgen}      & per-head policy & $\times$ & token & $\times$ & $\times$ (profiling) \\
PyramidKV~\cite{cai2024pyramidkv} & $(l,h,t)$, layer budget & $\times$ & token & $\sim$ (layer-wise) & \checkmark \\
Quest~\cite{tang2024quest}        & query/page & $\times$ & page & $\sim$ (query-dep.) & \checkmark \\
Ada-KV~\cite{feng2024adakv}       & $(l,h,t)$, adaptive budget & $\sim$ (budget-side) & token & \checkmark\ (adaptive) & \checkmark \\
EntroKV~\cite{entrokv2025}        & $(l,t)$, entropy layer budget & $\times$ & token & $\sim$ (layer-wise) & $\times$ (offline) \\
DuoAttention~\cite{xiao2024duoattention} & head class & $\sim$ (head-level) & head/token & $\times$ & $\times$ (offline) \\
\midrule
\multicolumn{6}{@{}l}{\emph{Channel-dimension and routing alternatives
  (included for completeness):}}\\
ThinK~\cite{liu2024think}         & channel (query-guided) & n/a & channel & $\sim$ (fixed ratio) & \checkmark \\
KIVI~\cite{liu2024kivi}           & n/a (2-bit K/V quant.) & n/a & n/a & n/a & \checkmark \\
QuaRot~\cite{ashkboos2024quarot}  & n/a (rot.\ + 4-bit) & n/a & n/a & n/a & \checkmark \\
InfLLM~\cite{xiao2024infllm}      & query--key block routing & $\sim$ (block-level) & block & $\sim$ & \checkmark\ (training-free) \\
\midrule
\method{} (ours)                  & $(l,g)$ group & \checkmark & page & \checkmark\ (\cref{thm:memory_guarantee}) & \checkmark (runtime entropy) \\
\bottomrule
\end{tabular}
\end{table*}

\subsection{Mechanistic head specialization}
Pruning studies showed early that many attention heads are
redundant~\cite{michel2019heads,voita2019analyzing}. Retrieval heads
carry long-context factuality
mechanistically~\cite{wu2024retrievalheads}. DuoAttention~%
\cite{xiao2024duoattention} statically classifies retrieval versus
streaming heads and applies heterogeneous compression---but it
classifies offline, at the granularity of the attention computation, and
needs a calibration corpus. \method{} weights heads \emph{within} each
physical GQA group at runtime, using measured sink-isolated entropy,
needs no calibration, and preserves the physical sharing structure of
GQA.

\section{Problem formulation and background}
\label{sec:background}

We first fix notation (\cref{tab:notation} in \cref{app:notation}), then
characterize the decode-phase memory wall, the physical layout of GQA
serving engines, and the structural metrics that quantify why
head-independent eviction fails.

\subsection{The decode-phase memory wall}
\label{subsec:roofline}

Decoding runs in two distinct execution regimes. Prefill processes the
input prompt via compute-bound matrix--matrix products. Autoregressive
decoding emits one token per step: self-attention over historical
tokens streams the key--value (KV) cache from high-bandwidth memory
(HBM) into on-chip SRAM, reducing the computation to memory-bound
matrix--vector products.

Let $\Bbatch$ denote batch size, $L$ transformer layers, $\HQ$ logical
query heads, $\HKV$ physical KV heads, $\Dhead$ head dimension, $T$ sequence length, and $b_{\mathrm{el}}$ bytes per element
($b_{\mathrm{el}}=2$ for FP16). Defining the GQA ratio
 $r \triangleq \HQ/\HKV$, the computational work per decode step is
 $W_{\mathrm{FLOPs}} = 4\,\Bbatch L\,\HQ\,\Dhead\,T$, while HBM data
traffic is
 $D_{\mathrm{bytes}} = 2\,b_{\mathrm{el}}\,\Bbatch L\,\HKV\,\Dhead\,T
+ \mathcal{O}(\Dmodel)$. The arithmetic intensity
 $\mathcal{A}_{\mathrm{decode}}$ is derived as:
\begin{equation}
\mathcal{A}_{\mathrm{decode}}
= \frac{W_{\mathrm{FLOPs}}}{D_{\mathrm{bytes}}}
= \frac{4\,\Bbatch L\,\HQ\,\Dhead\,T}
       {2\,b_{\mathrm{el}}\,\Bbatch L\,\HKV\,\Dhead\,T
        + \mathcal{O}(\Dmodel)}
\approx \frac{2r}{b_{\mathrm{el}}},
\label{eq:arithmetic_intensity}
\end{equation}
which yields $\mathcal{A}_{\mathrm{decode}} = 4.0$ FLOP/Byte for 16-bit
precision on an $r = 4$ architecture. The hardware ridge point on an
NVIDIA H100 SXM5 accelerator ($989.5$\TFLOPs{} dense FP16;
 $3.35$\TBs{} HBM3 bandwidth) is:
\begin{equation}
I_{\mathrm{ridge}}
= \frac{989.5 \times 10^{12}\ \text{FLOP/s}}
       {3.35 \times 10^{12}\ \text{Byte/s}}
\approx 295.4\ \text{FLOP/Byte},
\label{eq:ridge_point}
\end{equation}
demonstrating that decoding operates at $< 1.5\%$ of peak compute,
strictly memory-bandwidth bound. The total physical KV cache capacity
required is:
\begin{equation}
M_{\mathrm{KV}}(T)
= 4 \cdot \Bbatch \cdot L \cdot \HKV \cdot \Dhead \cdot T
\quad \text{bytes,}
\label{eq:kv_footprint}
\end{equation}
which for Llama-3.1-70B ($L{=}80$, $\HKV{=}8$, $\Dhead{=}128$) with
 $\Bbatch{=}8$ at $T{=}128\text{k}$ reaches $\approx 343.6$\GB. These
figures are architectural constants plus arithmetic; they motivate the
compression problem this paper addresses.

\subsection{GQA mechanics and paged memory layout}
\label{subsec:gqa}

Let $\mathbf{X} \in \mathbb{R}^{T \times \Dmodel}$ denote token
representations. For query head $h \in \{0, \dots, \HQ-1\}$, mapped to
physical group $g(h) = \lfloor h/r \rfloor$, the attention probability
vector at a given decode step is:
\begin{equation}
\mathbf{S}_{l,h}
= \frac{\mathbf{q}_{l,h}\,\mathbf{K}_{l,g(h)}^{\top}}{\sqrt{\Dhead}},
\qquad
\Attn{l}{h} = \softmax\!\bigl(\mathbf{S}_{l,h}\bigr)
\in \Delta^{T-1},
\label{eq:gqa_attention}
\end{equation}
where $\Delta^{T-1}$ is the probability simplex. MHA corresponds to
 $r=1$~\cite{vaswani2017attention}, MQA to $\HKV=1$~\cite{shazeer2019fast},
and GQA to $1<\HKV<\HQ$~\cite{ainslie2023gqa}. In vLLM, the physical KV
cache is partitioned into non-contiguous page blocks $P_b \in \PageSet$ of fixed capacity $B$ tokens:
\begin{equation}
\mathrm{Shape}\bigl(\mathbf{K}_{\mathrm{physical}}\bigr)
= \bigl[\Nblk,\ \HKV,\ B,\ \Dhead\bigr].
\label{eq:paged_layout}
\end{equation}
A logical block table maps $(l,g,\lfloor t/B\rfloor)\mapsto P_b$. Fused
CUDA kernels access memory per $(l,g)$ tuple; eviction at sub-page
granularity therefore causes internal fragmentation or cache-unrolling
overhead.

\subsection{Quantifying the structural breakdown}
\label{subsec:uor}

\begin{definition}[Retention set and union overhead ratio]
\label{def:uor}
Let
 $\TokSet{h} \triangleq \argTopK_{j \in \{1,\dots,T\}}
 \bigl(\Attn{l,h,j},\, k\bigr)$ be the set of $k$ token indices selected by head $h$ under per-head
budget $k \le T$, with ties broken deterministically by token index so
that $\TokSet{h}$ is a single well-defined set. For physical group $g$,
the union overhead ratio under head-independent selection is:
\begin{equation}
\UOR(g, k)
\triangleq
\frac{\bigl|\bigcup_{h \in \Group{g}} \TokSet{h}\bigr|}{k},
\qquad
1.0 \le \UOR(g, k) \le r.
\label{eq:uor_def}
\end{equation}
\end{definition}

\begin{definition}[Intra-group Jaccard disagreement index]
\label{def:jaccard}
The pairwise disagreement between sibling heads
 $h_i, h_j \in \Group{g}$ is:
\begin{equation}
\mathcal{D}_{\mathrm{Jaccard}}(h_i, h_j)
\triangleq 1 -
\frac{|\TokSet{h_i} \cap \TokSet{h_j}|}
     {|\TokSet{h_i} \cup \TokSet{h_j}|}.
\label{eq:jaccard_def}
\end{equation}
\end{definition}

The two metrics are not independent. For a two-head group they determine
each other exactly; for general $r$, they sandwich the inflation
two-sidedly (\cref{thm:sandwich_bound}). This gives
\cref{def:uor,def:jaccard} operational content: measured Jaccard
disagreement among sibling heads \emph{implies} a quantifiable memory
penalty for any head-independent policy.

\begin{proposition}[Exact UOR--disagreement identity, $r = 2$]
\label{prop:uorjaccard}
Let $r = 2$ and
 $D = \mathcal{D}_{\mathrm{Jaccard}}(h_1, h_2)$. Then:
\begin{equation}
\UOR(g, k) = \frac{2}{2 - D}.
\label{eq:uor_identity}
\end{equation}
\end{proposition}

\begin{proof}
Let $U = \TokSet{h_1} \cup \TokSet{h_2}$ and
 $I = \TokSet{h_1} \cap \TokSet{h_2}$. By inclusion--exclusion,
 $|U| = 2k - |I|$. By definition $D = 1 - |I|/|U|$, so
 $|I| = (1-D)|U|$. Thus
 $|U| = 2k - (1-D)|U| \implies (2-D)|U| = 2k$. Dividing by $k$ yields
\cref{eq:uor_identity}.
\end{proof}

\begin{remark}[Response curve of the identity]
\label{rem:phi}
The map $\phi(D) \triangleq 2/(2-D)$ from disagreement to overhead is
strictly increasing, $\phi'(D) = 2/(2-D)^2 > 0$, and strictly convex,
 $\phi''(D) = 4/(2-D)^3 > 0$, on $[0,1]$: the memory penalty accelerates
as sibling selections diverge. Reference points: $D = 0$ (complete
consensus) gives $\UOR = 1.0$; $D = 0.5$ gives $\UOR = 1.33$; $D = 2/3$ gives the exact $1.5\times$ budget overshoot; $D = 1$ (orthogonal
attention) gives the worst case $\UOR = 2.0$ for $r = 2$.
\end{remark}

\begin{corollary}[Elementary Bonferroni certificate]
\label{cor:bonferroni}
Let $J_{ij} = 1 - \mathcal{D}_{\mathrm{Jaccard}}(h_i, h_j)$ and
 $\bar{J} = \frac{2}{r(r-1)} \sum_{i < j} J_{ij}$. By Bonferroni's
inequality:
\begin{equation}
\UOR(g, k) \ge \max\Bigl\{ 1, \; r - r(r-1)\bar{J} \Bigr\}.
\label{eq:bonferroni_bound}
\end{equation}
\cref{eq:bonferroni_bound} certifies a twofold inflation
($\UOR \ge 2$) whenever $\bar{J} \le \frac{r-2}{r(r-1)}$---yielding
 $\bar{J} \le 1/6$ for $r=4$ and $\bar{J} \le 3/28$ for $r=8$.
\end{corollary}

\begin{proof}
 $|\bigcup_h \TokSet{h}| \ge \sum_h |\TokSet{h}|
-\sum_{i<j}|\TokSet{h_i}\cap\TokSet{h_j}|$ by Bonferroni, and
 $|\TokSet{h_i}\cap\TokSet{h_j}| = J_{ij}\,|\TokSet{h_i}\cup\TokSet{h_j}|
\le 2kJ_{ij}$. Summing over the $r(r-1)/2$ pairs gives
 $\sum_{i<j}|\TokSet{h_i}\cap\TokSet{h_j}| \le k\,r(r-1)\bar{J}$, and
dividing by $k$ yields the bound; the clamp is $\UOR \ge 1$ from
\cref{def:uor}. The threshold follows by solving
 $r - r(r-1)\bar{J} \ge 2$ for $\bar{J}$.
\end{proof}

\begin{theorem}[GQA inflation sandwich bound]
\label{thm:sandwich_bound}
For general $r \ge 2$, with $J_{ij}$ and
 $\bar{J} = \frac{2}{r(r-1)}\sum_{i<j}J_{ij}$ as in
\cref{cor:bonferroni}, the union overhead ratio is bounded by:
\begin{equation}
\frac{r}{1 + 2(r-1)\dfrac{\bar{J}}{1+\bar{J}}}
\;\le\;
\UOR(g, k)
\;\le\;
\min\Biggl(r, \; r - \max_{i < j} \frac{2 J_{ij}}{1 + J_{ij}}\Biggr).
\label{eq:sandwich_bound}
\end{equation}
\end{theorem}

\begin{proof}
\emph{Lower bound.} Let
 $\mathcal{U} = \bigcup_{i=1}^{r}\TokSet{h_i}$ and define the coverage
degree $x_m \triangleq \sum_{i=1}^{r}
\mathbf{1}_{\{m \in \TokSet{h_i}\}}$ for $m \in \mathcal{U}$. Exchanging
the order of summation and splitting the double sum into diagonal and
off-diagonal terms,
\begin{equation}
\sum_{m} x_m = rk,
\qquad
\sum_{m} x_m^{2}
= \sum_{i=1}^{r}\sum_{j=1}^{r}
  |\TokSet{h_i} \cap \TokSet{h_j}|
= rk + 2\sum_{i<j} I_{ij},
\label{eq:coverage_sums}
\end{equation}
where $I_{ij} = |\TokSet{h_i} \cap \TokSet{h_j}|$. Cauchy--Schwarz on
 $(x_m)$ and the all-ones vector gives
 $(rk)^{2} \le |\mathcal{U}|\,(rk + 2\sum_{i<j}I_{ij})$, hence
 $\UOR \ge r^{2}/(r + \tfrac{2}{k}\sum_{i<j}I_{ij})$. From
 $J_{ij} = I_{ij}/(2k - I_{ij})$ one gets
 $I_{ij} = 2kJ_{ij}/(1+J_{ij})$, so
 $\tfrac{2}{k}\textstyle\sum_{i<j}I_{ij}
= 4\sum_{i<j}\frac{J_{ij}}{1+J_{ij}}$. Since
 $f(x) = \frac{x}{1+x}$ is strictly concave
($f''(x) = -\frac{2}{(1+x)^{3}} < 0$), Jensen's inequality gives
 $\sum_{i<j} f(J_{ij}) \le \frac{r(r-1)}{2}f(\bar{J})$; substituting this
upper bound into the denominator yields the lower bound of
\cref{eq:sandwich_bound}.

\emph{Upper bound.} For any pair $\{h_a, h_b\}$,
 $\mathcal{U} = (\TokSet{h_a} \cup \TokSet{h_b}) \cup
\bigl(\bigcup_{m \neq a,b}\TokSet{h_m}\bigr)$, so sub-additivity and the
map above give
 $|\mathcal{U}| \le (2k - I_{ab}) + (r-2)k = rk - I_{ab}$, i.e.,
 $\UOR \le r - \frac{2J_{ab}}{1+J_{ab}}$. The pair is arbitrary, so the
bound holds with the maximizing pair; the outer $\min(r,\cdot)$ is the
natural boundary from \cref{def:uor}.
\end{proof}

The full eight-step derivation of \cref{thm:sandwich_bound} appears in
\cref{app:sandwich}.

\paragraph{Tightness and threshold synchronization.}
The lower bound is tight at both extremes: at $\bar{J}=0$ (pairwise
disjoint selections) it evaluates to $r$, the exact worst case, and at
 $\bar{J}=1$ (identical selections) to $1$, the exact best case. The
upper bound is tight at the disjoint extreme; at perfect alignment it
evaluates to $r-1$ for $r>2$---mathematically valid, though loose
relative to the true limit of $1.0$, because the pairwise profile does
not see higher-order overlap structure.
\Cref{tab:certificates} compares the two certificates.

\begin{table}[t]
\centering\small
\caption{The two disagreement certificates compared. The sandwich
  certificate of \cref{thm:sandwich_bound} strictly widens the
  certified-inflation regime; it is also naturally tight at consensus,
  where the Bonferroni form needs its clamp.}
\label{tab:certificates}
\begin{tabular}{@{}lcc@{}}
\toprule
 & Bonferroni & Sandwich \\
 & (\cref{cor:bonferroni}) & (\cref{thm:sandwich_bound}) \\
\midrule
 $2\times$ threshold ($r{=}4$) &
 $\bar{J} \le \frac{1}{6} \approx 0.167$ &
 $\bar{J} \le \frac{1}{5} = 0.200$ \\
 $2\times$ threshold ($r{=}8$) &
 $\bar{J} \le \frac{3}{28} \approx 0.107$ &
 $\bar{J} \le \frac{3}{11} \approx 0.273$ \\
At $\bar{J} \to 0$ (disjoint) & $r$ (exact, tight) &
 $r$ (exact, tight) \\
At $\bar{J} \to 1$ (consensus) & clamp to $1$ & $1$ (naturally tight) \\
\bottomrule
\end{tabular}
\end{table}

Under \cref{thm:sandwich_bound}, twofold inflation ($\UOR \ge 2$) is
certified up to $\bar{J} \le \frac{r-2}{3r-2}$---yielding
 $\bar{J} \le 1/5$ for $r=4$ and $\bar{J} \le 3/11$ for $r=8$---a
strictly wider certification range than the Bonferroni threshold of
\cref{cor:bonferroni}. When siblings fully disagree ($\bar{J} \to 0$),
they attend to essentially disjoint token sets and $\UOR \to r$: a
 $4\times$ cache inflation for Llama-3.1-8B ($r=4$) and an $8\times$ inflation for $r=8$ architectures.

\section{The \method{} methodology}
\label{sec:methodology}

\method{} aligns KV-cache eviction with the hardware-physical structure
of GQA serving. \Cref{alg:pageentrokv} states the complete decision
procedure. \Cref{subsec:entropy} fixes the entropy machinery behind the
pooling weights, \cref{subsec:pooling} derives group pooling, page
mapping, and eviction, and \cref{subsec:systems} covers integration in
vLLM.

\subsection{Sink-isolated discrete R\'enyi attention entropy}
\label{subsec:entropy}

The pooling weights of \method{} come from a discrete entropy family
evaluated on sink-isolated attention distributions. This subsection
fixes the definitions and the suppression guarantee the rest of the
paper relies on.

\begin{definition}[Discrete generalized R\'enyi attention-entropy]
\label{def:renyi}
For $\Attn{l}{h} \in \Delta^{T-1}$ and
 $\alpha \in (0,1) \cup (1,\infty)$:
\begin{equation}
\Renyi{\alpha}{\Attn{l}{h}}
\triangleq \frac{1}{1-\alpha}
\ln \Biggl( \sum_{t=1}^T \Attn{l,h}(t)^{\alpha} \Biggr).
\label{eq:renyi_def}
\end{equation}
\end{definition}

Writing $\mathbf{A}\equiv\Attn{l}{h}$ for brevity, taking limits
establishes the entropy family:
\begin{align}
\lim_{\alpha \to 1} \Renyi{\alpha}{\mathbf{A}}
&= -\sum_{t=1}^T \mathbf{A}(t) \ln \mathbf{A}(t)
&& \text{(Shannon)}, \label{eq:shannon_limit} \\
\Collision{\mathbf{A}}
&= -\ln \sum_{t=1}^T \mathbf{A}(t)^2
 = -\ln \|\mathbf{A}\|_2^2
&& \text{(collision, } \alpha{=}2\text{)}, \label{eq:collision_entropy} \\
\mathcal{H}_{\infty}(\mathbf{A})
&= -\ln \max_t \mathbf{A}(t)
&& \text{(min-entropy)}. \label{eq:min_entropy}
\end{align}
We fix $\alpha=2$ for three reasons: collision entropy is a single inner
product per head, it is numerically stable in FP16 (\cref{rem:fp16}),
and it amplifies attention spikes while suppressing background tail
mass. \Cref{tab:alpha_primitives} makes the computational contrast
concrete; the limit-order proofs appear in \cref{app:renyi}.

\begin{table}[t]
\centering\small
\caption{Why $\alpha = 2$: on-chip cost and numerical behavior of the
  three entropy orders for one head with $T$ positions. Collision
  entropy is the only order that reduces to a single dot product with a
  bounded logarithm argument.}
\label{tab:alpha_primitives}
\begin{tabular}{@{}p{0.21\columnwidth}p{0.24\columnwidth}
                 p{0.26\columnwidth}p{0.17\columnwidth}@{}}
\toprule
 & Shannon ($\alpha{\to}1$) & Collision ($\alpha{=}2$) &
 Min-entropy ($\alpha{\to}\infty$) \\
\midrule
Definition & $-\sum_t A\ln A$ & $-\ln\|A\|_2^2$ & $-\ln\max_t A$ \\
On-chip primitive & $T$ ln-calls & \textbf{1 dot product} $+\,1$ ln &
 $T$-way max $+\,1$ ln \\
FP16 stability & cancellation as $A(t)\to 0$ & \textbf{stable};
 argument in $(0,1]$ & plateau ties \\
Spike sensitivity & logarithmic scaling & \textbf{quadratic
 amplification} of spikes & ignores shape below peak \\
\bottomrule
\end{tabular}
\end{table}

\begin{definition}[Sink-isolated collision entropy]
\label{def:sinkentropy}
Let $\Ssink$ denote the set of attention-sink tokens, taken as the first
 $|\Ssink| = 4$ token positions, following the convention of
StreamingLLM~\cite{xiao2024streamingllm}. Under the non-degeneracy
hypothesis $\sum_{t' \in \Ssink}\Attn{l,h}(t') < 1$ (a head concentrated
entirely on $\Ssink$ is handled by the implementation clamp;
\cref{rem:fp16}), the sink-isolated attention distribution
 $\tilde{\mathbf{A}}_{l,h}$ on non-sink positions $t \notin \Ssink$ is:
\begin{equation}
\tilde{\mathbf{A}}_{l,h}(t)
\triangleq \frac{\Attn{l,h}(t)}
{1 - \sum_{t' \in \Ssink} \Attn{l,h}(t')},
\qquad t \notin \Ssink,
\label{eq:sink_isolated_dist}
\end{equation}
and the sink-isolated collision entropy is
 $\CollisionSink{\Attn{l}{h}} \triangleq
-\ln \|\tilde{\mathbf{A}}_{l,h}\|_2^2$.
\end{definition}

Without isolation, an attention-sink head that concentrates most of its
mass on $\Ssink$ has very low raw collision entropy and would receive
the same large Boltzmann weight as a genuine retrieval head, misdirecting
the pooled mass toward the sink tokens. The next theorem makes the
separation quantitative: isolation drives the sink head's weight to
zero at a controllable rate while the retrieval head's weight stays
bounded away from zero.

\begin{theorem}[Sink head suppression]
\label{thm:sink_suppression}
Let $N = T - |\Ssink|$ and $\tau_g > 0$. Consider two head types:
\emph{(i)} a pure sink head concentrating mass $1-\delta$ on $\Ssink$ (with $0<\delta<1$) and a uniform residual over the $N$ non-sink
positions; \emph{(ii)} a retrieval head with mass $1-\epsilon'$ at a
non-sink needle $t^{*}$ (with $0<\epsilon'<1$) and zero sink mass.
Under sink isolation, the sink head has
 $\|\tilde{\mathbf{A}}\|_2^2 = 1/N$, hence
 $\CollisionSink{} = \ln N$ exactly, and its Boltzmann factor is:
\begin{equation}
e^{-\CollisionSink{}/\tau_g}
= e^{-\ln(N)/\tau_g}
= N^{-1/\tau_g}
= \frac{1}{(T - |\Ssink|)^{1/\tau_g}}
\;\xrightarrow[T \to \infty]{}\; 0,
\label{eq:sink_decay}
\end{equation}
whereas the retrieval head maintains
 $\CollisionSink{\Attn{l}{h'}} = -2\ln(1-\epsilon') = \mathcal{O}(1)$ and factor $(1-\epsilon')^{2/\tau_g} > 0$: its weight is bounded away
from zero for every $T$.
\end{theorem}

\begin{proof}
For the sink head, the residual is uniform:
 $\tilde{\mathbf{A}}_{l,h}(t) = (\delta/N)/\delta = 1/N$ on all $N$ non-sink positions, so $\|\tilde{\mathbf{A}}\|_2^2 = N \cdot N^{-2} = 1/N$ and $\CollisionSink{} = -\ln(1/N) = \ln N$; its factor is
 $e^{-\ln N/\tau_g} = N^{-1/\tau_g} \to 0$. For the retrieval head, zero
sink mass makes isolation the identity, so
 $\|\tilde{\mathbf{A}}_{l,h'}\|_2^2 \ge (1-\epsilon')^{2}$ and the factor
is at least $e^{2\ln(1-\epsilon')/\tau_g} = (1-\epsilon')^{2/\tau_g}$.
For contrast, the \emph{raw} collision entropy of the sink head tends
to $-2\ln(1-\delta)$---also $\mathcal{O}(1)$---so raw entropy cannot
separate the two types; the separation is created by the isolation
step, not by the entropy family alone.
\end{proof}

\begin{remark}[Why isolation, not the entropy family, does the work]
\label{rem:necessity}
For the pure sink head of \cref{thm:sink_suppression}, the two regimes
contrast sharply. \emph{Raw} collision entropy is evaluated over all
 $T$ tokens, where $\|A\|_2^2 \ge (1-\delta)^2 \approx 1$ and hence
 $\mathcal{H}_2 \approx 0$: the sink head receives an $\mathcal{O}(1)$ Boltzmann weight and masquerades as a sharp retrieval head, stealing
group attention mass. \emph{Sink-isolated} entropy is evaluated strictly
over the $N$ non-sink tokens, where $\|\tilde{A}\|_2^2 = 1/N \to 0$ and
hence $\Hsink = \ln N \to \infty$: the weight decays as $N^{-1/\tau_g}$ and the head is suppressed. This holds for every order $\alpha$, since
the raw power sum $\sum_t A(t)^\alpha \ge (1-\delta)^\alpha = \mathcal{O}(1)$ for any fixed $\alpha$---no choice of entropy order separates the two
head types on the raw distribution. The separation is created entirely
by the conditional renormalization of \cref{eq:sink_isolated_dist}.
\end{remark}

\begin{definition}[Prefill observation window]
\label{def:obswindow}
The sink-isolated collision entropies used for head weighting are
evaluated over the prefill observation window $W_{\mathrm{obs}}$, the
final $\min(T, W_{\max})$ tokens of the prompt with $W_{\max}=4096$ (the observation-window convention of SnapKV~\cite{li2024snapkv}):
\begin{equation}
\CollisionSink{\Attn{l}{h}\upharpoonright W_{\mathrm{obs}}}
= -\ln \bigl\| \tilde{\mathbf{A}}_{l,h} \upharpoonright
W_{\mathrm{obs}} \bigr\|_2^2,
\label{eq:window_entropy}
\end{equation}
where the restriction is renormalized over the window (the
normalization constant is part of the configuration and is recorded
with it). The entropies---and hence the pooling weights of
\cref{eq:boltzmann_weights}---are computed once, at the end of prefill,
and remain fixed for the entire auto-regressive generation lifecycle.
\end{definition}

\subsection{Hardware-aligned group pooling and paged mapping}
\label{subsec:pooling}

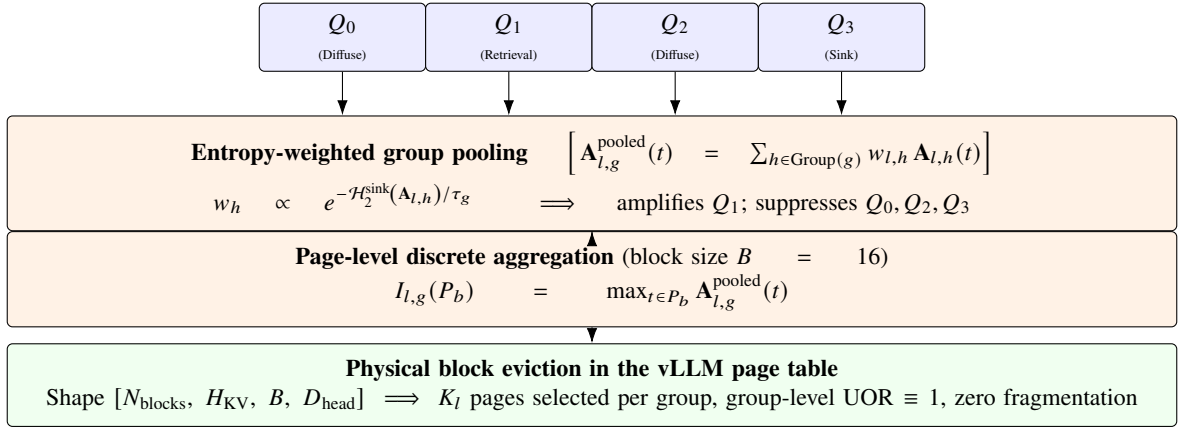
\begin{figure*}[t]
\centering
\begin{tikzpicture}[font=\small, node distance=5mm]
  \node[qbox] (q0) at (-33mm, 0) {$Q_0$\\{\tiny(Diffuse)}};
  \node[qbox] (q1) at (-11mm, 0) {$Q_1$\\{\tiny(Retrieval)}};
  \node[qbox] (q2) at ( 11mm, 0) {$Q_2$\\{\tiny(Diffuse)}};
  \node[qbox] (q3) at ( 33mm, 0) {$Q_3$\\{\tiny(Sink)}};

  \node[pbox] (pool) at (0, -18mm)
    {\textbf{Entropy-weighted group pooling}
     $\;\Bigl[\,\AttnPool{l}{g}(t)
     = \sum_{h\in\Group{g}}w_{l,h}\,\Attn{l}{h}(t)\Bigr]$\\[2pt]
     $w_h \propto e^{-\CollisionSink{\Attn{l}{h}}/\tau_g}
     \;\Longrightarrow\;$ amplifies $Q_1$; suppresses $Q_0,Q_2,Q_3$};

  \foreach \q in {q0,q1,q2,q3}
    \draw[arr] (\q.south) -- (\q.south |- pool.north);

  \node[pbox] (page) at (0, -32mm)
    {\textbf{Page-level discrete aggregation} (block size $B=16$)\\[2pt]
     $\PageImp{l}{g}{b} = \max_{t\in P_b}\AttnPool{l}{g}(t)$};

  \draw[arr] (pool.south) -- (page.north);

  \node[pbox, fill=green!6] (vllm) at (0, -46mm)
    {\textbf{Physical block eviction in the vLLM page table}\\
     Shape $[\Nblk,\ \HKV,\ B,\ \Dhead]$      $\;\Longrightarrow\;$ $K_l$ pages selected per group,
     group-level $\UOR \equiv 1$, zero fragmentation};

  \draw[arr] (page.south) -- (vllm.north);
\end{tikzpicture}
\caption{\method{} mechanism for an $r=4$ architecture such as
  Llama-3.1-8B: entropy-weighted pooling with sink-isolated entropy
  amplifies the retrieval head $Q_1$, suppresses diffuse heads $Q_0$,
  $Q_2$, and the attention-sink head $Q_3$, and maps the result onto
  physical PagedAttention blocks.}
\label{fig:mechanism}
\end{figure*}

\paragraph{The union overhead dilemma.}
If selection occurs independently per query head with
 $\TokSet{h}=\argTopK(\Attn{l,h},k)$, the engine must retain the union
 $\bigcup_{h\in\Group{g}}\TokSet{h}$, with $\UOR(g,k)\in[1,r]$ as defined
in \cref{eq:uor_def}, exactly determined by disagreement at $r=2$ (\cref{prop:uorjaccard}) and sandwiched at every ratio
(\cref{thm:sandwich_bound}): an $800\%$ memory expansion for
 $r=8$ architectures when selections diverge, causing out-of-memory
faults.

\paragraph{Needle dilution under arithmetic mean pooling.}
Naive mean pooling assigns
\begin{equation}
\AttnMean{l}{g}(t^{*})
= \frac{1}{r}\!\left[(1-\epsilon)+(r-1)\frac{1}{T}\right]
\xrightarrow[T\to\infty]{} \frac{1-\epsilon}{r}.
\label{eq:mean_dilution}
\end{equation}
The needle's pooled mass drops by exactly $1-1/r$---$87.5\%$ for
 $r=8$, independent of $\epsilon$; for $\epsilon=0.05$ the pooled score
is $(1-\epsilon)/8=0.11875$, below common eviction thresholds.

\paragraph{Entropy-weighted group pooling.}
\method{} replaces mean pooling with entropy-weighted group pooling:
\begin{equation}
w_{l,h}
= \frac{\exp\bigl(-\CollisionSink{\Attn{l}{h}}/\tau_g\bigr)}
       {\sum_{j \in \Group{g}}
        \exp\bigl(-\CollisionSink{\Attn{l}{j}}/\tau_g\bigr)},
\qquad
\textstyle\sum_{h \in \Group{g}} w_{l,h} = 1,
\label{eq:boltzmann_weights}
\end{equation}
\begin{equation}
\AttnPool{l}{g}(t)
= \sum_{h \in \Group{g}} w_{l,h}\,\Attn{l,h}(t).
\label{eq:pooled_attention}
\end{equation}
Low-entropy retrieval heads receive exponentially larger weights;
diffuse streaming heads and attention-sink heads are suppressed, the
latter by \cref{def:sinkentropy,thm:sink_suppression}. Pooling runs
once per physical KV group---matching how GQA shares memory. The pooled
output stays a convex combination of probability vectors, and that
property is exactly what the budget guarantee of
\cref{thm:memory_guarantee} requires. With $\Ssink=\emptyset$,
\cref{eq:boltzmann_weights} reduces to raw collision entropy. The
temperature $\tau_g$ interpolates the entire family of group
aggregators; we characterize it exactly:

\begin{proposition}[Temperature monotonicity and limits]
\label{prop:temperature_limits}
Let $w_{l,h}(\tau_g)$ be defined by \cref{eq:boltzmann_weights}.
\textup{(i)} $w_{l,h}(\tau_g)$ is non-increasing in
 $\CollisionSink{\Attn{l}{h}}$ (strictly decreasing when the entropies
are distinct), for every $\tau_g>0$.
\textup{(ii)} $\lim_{\tau_g \to \infty} w_{l,h}(\tau_g) = 1/r$ (arithmetic mean pooling).
\textup{(iii)} $\lim_{\tau_g \to 0^+} w_{l,h}(\tau_g)$ is the uniform
weight on the set of entropy minimizers---one-hot for a unique
minimizer.
\end{proposition}

\begin{proof}
 $w$ is a softmax of the scores
 $-\CollisionSink{\Attn{l}{h}}/\tau_g$, which is strictly increasing in
its score; hence $w$ is non-increasing in the entropy (strictly under
distinct values). (ii) The entropies are bounded by $\ln T$, so every
exponential tends to $1$ and $w \to 1/r$. (iii) Multiplying numerator
and denominator by $e^{\mathcal{H}_{\min}/\tau_g}$, every exponent is
 $\le 0$, equals $0$ exactly at the minimizers, and
 $e^{-c/\tau_g} \to 0$ for each $c>0$.
\end{proof}

\Cref{prop:temperature_limits} also motivates the failure-mode fallback
of \cref{sec:discussion}: when intra-group entropies are
indistinguishable, \cref{prop:temperature_limits}~(i) loses its bite,
the entropy signal carries no information, and the operator can switch
to max-over-heads aggregation without changing the budget
guarantee---which depends only on pooling \emph{before} selection
(\cref{thm:memory_guarantee}).

\paragraph{Discrete physical block-frame aggregation.}
Token scores are aggregated onto physical page frames
 $P_b \in \PageSet$ via max-reduction, and the eviction operator selects
the top $K_l$ pages at the hardware tuple $(l,g,P_b)$:
\begin{equation}
\PageImp{l}{g}{b}
= \max_{t \in P_b} \AttnPool{l}{g}(t),
\qquad
\RetSet{l}{g}
= \mathrm{TopK}\Bigl(\{I_{l,g}(P_b)\}_{P_b \in \PageSet},\ K_l\Bigr).
\label{eq:page_aggregation}
\end{equation}
(Page-level retention sets $\RetSet{l}{g}=\mathcal{K}_{l,g}$ are
distinguished from the token-level selection sets $\TokSet{h}$ of
\cref{def:uor} by their group subscript.)

\paragraph{Why block-max rather than block-mean reduction.}
Aggregating page importance by the within-page mean instead of the max
dilutes exactly the pages the framework exists to protect. With
 $B=16$, a page holding a single needle token of pooled mass $0.95$ and
fifteen near-zero tokens scores $0.95/16 = 0.0594$ under block-mean---%
a $93.75\%$ in-page dilution---while block-max preserves $0.95$. A
diffuse page of sixteen $0.06$-mass tokens then outranks it
($0.06 > 0.0594$) and evicts the needle page. Block-max is also exactly
scale-invariant in $B$: it preserves the peak coordinate regardless of
page size, whereas block-mean attenuates every sharp page by $1/B$.
These figures are illustrative arithmetic, not measurements;
\cref{lem:page_monotonicity} makes the corresponding guarantee formal.

\begin{proposition}[Budget initialization and exact conservation]
\label{prop:budget_conservation}
With total budget $\mathcal{B}_{\mathrm{total}}$, the per-layer token
budget $\mathcal{B}_l$ and static per-group page quota $K_l$ are:
\begin{equation}
\mathcal{B}_l = \frac{\mathcal{B}_{\mathrm{total}}}{L},
\qquad
K_l = \Bigl\lfloor \frac{\mathcal{B}_l}{\HKV \cdot B} \Bigr\rfloor.
\label{eq:page_quotas}
\end{equation}
\textup{(i)} Under the floor quotas alone, exact conservation
($\sum_l \HKV B\,K_l = \mathcal{B}_{\mathrm{total}}$) holds iff
 $\mathcal{B}_{\mathrm{total}} \equiv 0 \pmod{L\,\HKV\,B}$.
\textup{(ii)} With largest-remainder (Hamilton) rounding over the
fractional remainders
 $R_l = \frac{\mathcal{B}_l}{\HKV B} - K_l \in [0,1)$---allocating $+1$ page to the largest-remainder layer-groups---exact conservation holds
iff $\HKV\,B$ divides $\mathcal{B}_{\mathrm{total}}$, in which case at
most $L-1$ extra page-groups suffice.
\textup{(iii)} Otherwise the residual
 $\mathcal{B}_{\mathrm{total}} \bmod \HKV\,B < \HKV\,B$ tokens---less
than one page-group---is unavoidable at page granularity.
\end{proposition}

\begin{proof}
Page allocation is quantized in $\HKV B$ tokens per layer-group, so
 $\sum_l \HKV B\,K_l$ is always a multiple of $\HKV B$; it equals
 $\mathcal{B}_{\mathrm{total}}$ iff $\HKV B$ divides
 $\mathcal{B}_{\mathrm{total}}$, which proves (ii)'s necessity and
(iii). For (i): with uniform budgets,
 $\mathcal{B}_l/(\HKV B) = \mathcal{B}_{\mathrm{total}}/(L\,\HKV B)$, so
all $L$ floors sum to $\mathcal{B}_{\mathrm{total}}$ iff the quotient is
an integer. For (ii)'s sufficiency: writing
 $x = \mathcal{B}_{\mathrm{total}}/(\HKV B) \in \mathbb{N}$, the floor
sum is $L\lfloor x/L\rfloor$ and the deficit $x - L\lfloor x/L\rfloor
\in \{0,\dots,L-1\}$ is distributed as $+1$ allocations to the
largest-remainder layers, restoring $\sum_l K_l = x$.
\end{proof}

The full derivation, including worked numerical cases on Llama-3.1-8B
geometry, appears in \cref{app:budget}.

\paragraph{Page-level retention across groups.}
Because PagedAttention blocks are shared across the $\HKV$ groups of a
sequence, per-group retention sets are materialized by the engine as
their per-layer union. \Cref{prop:engine_page_accounting} gives the
exact accounting.

\begin{proposition}[Engine-level page accounting]
\label{prop:engine_page_accounting}
Let $P^{\mathrm{ret}}_l = \bigcup_g \RetSet{l}{g}$ be the layer page
union, where each $\RetSet{l}{g}$ is given by
\cref{eq:page_aggregation}. Then, unconditionally:
\begin{equation}
K_l \le \bigl| P^{\mathrm{ret}}_l \bigr| \le \HKV\, K_l.
\label{eq:engine_union_bounds}
\end{equation}
The lower bound is attained iff all groups select identical page sets;
the upper bound iff selections are pairwise disjoint.
\end{proposition}

\begin{proof}
By \cref{eq:page_aggregation}, each group retains exactly $K_l$ pages
(exact TopK cardinality), so $|\RetSet{l}{g}|=K_l$ for every $g$; the
union contains any single member, giving the lower bound. The union of
 $\HKV$ sets of cardinality $K_l$ each has cardinality at most
 $\HKV\,K_l$. The equality conditions are immediate.
\end{proof}

The page-union ratio, taken relative to the per-group retention
cardinality $|P^{\mathrm{ret}}_l|/K_l$, lies in $[1,\ \HKV]$ and equals
one whenever the per-group pooled page profiles coincide.
\Cref{prop:engine_page_accounting} states the guarantee at the level
where it holds---per layer, across the $\HKV$ groups, on top of the
per-group guarantee of \cref{thm:memory_guarantee}---and quantifies the
residual overhead exactly; the pilot measures it in
\cref{subsec:pilot-ablations}. The full proof, including both equality
conditions and the page-union ratio formalism, appears in
\cref{app:pageaccounting}.

\subsection{Systems integration in vLLM}
\label{subsec:systems}

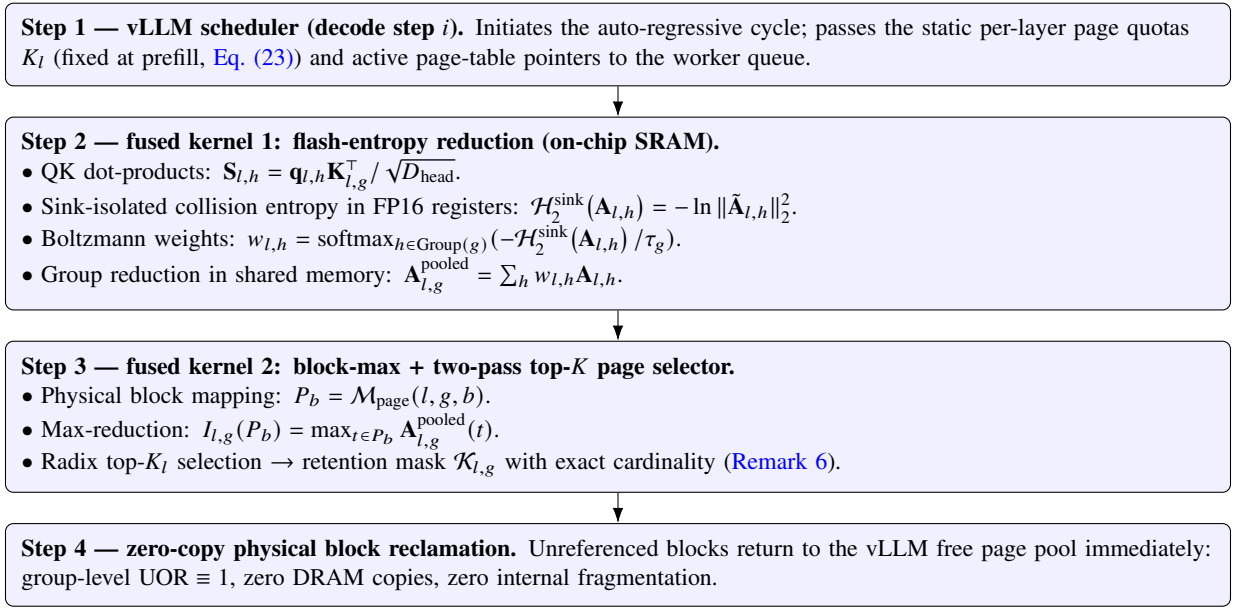
\begin{figure*}[t]
\centering
\begin{tikzpicture}[font=\small, node distance=4mm]
  \node[flow] (step1)
    {\textbf{Step 1 --- vLLM scheduler (decode step $i$).} Initiates the
     auto-regressive cycle; passes the static per-layer page quotas
     $K_l$ (fixed at prefill, \cref{eq:page_quotas}) and active
     page-table pointers to the worker queue.};
  \node[flow, below=of step1] (step2)
    {\textbf{Step 2 --- fused kernel 1: flash-entropy reduction
     (on-chip SRAM).}\\
     \textbullet\ QK dot-products:
       $\mathbf{S}_{l,h}=\mathbf{q}_{l,h}\mathbf{K}_{l,g}^{\top}/
        \sqrt{\Dhead}$.\\
     \textbullet\ Sink-isolated collision entropy in FP16 registers:
       $\CollisionSink{\Attn{l}{h}}=-\ln\|\tilde{\mathbf{A}}_{l,h}\|_2^2$.\\
     \textbullet\ Boltzmann weights:
       $w_{l,h}=\mathrm{softmax}_{h\in\Group{g}}
        (-\CollisionSink{\Attn{l}{h}}/\tau_g)$.\\
     \textbullet\ Group reduction in shared memory:
       $\AttnPool{l}{g}=\sum_h w_{l,h}\Attn{l}{h}$.};
  \node[flow, below=of step2] (step3)
    {\textbf{Step 3 --- fused kernel 2: block-max + two-pass top-$K$      page selector.}\\
     \textbullet\ Physical block mapping:
       $P_b = \mathcal{M}_{\mathrm{page}}(l,g,b)$.\\
     \textbullet\ Max-reduction:
       $\PageImp{l}{g}{b}=\max_{t\in P_b}\AttnPool{l}{g}(t)$.\\
     \textbullet\ Radix top-$K_l$ selection $\to$ retention mask
       $\RetSet{l}{g}$ with exact cardinality (\cref{rem:topk}).};
  \node[flow, below=of step3] (step4)
    {\textbf{Step 4 --- zero-copy physical block reclamation.}
     Unreferenced blocks return to the vLLM free page pool immediately:
     group-level $\UOR \equiv 1$, zero DRAM copies, zero internal
     fragmentation.};
  \foreach \a/\b in {step1/step2, step2/step3, step3/step4}
    \draw[arr] (\a.south) -- (\b.north);
\end{tikzpicture}
\caption{Hardware-level execution flow of the fused \method{} kernels
  within the vLLM decoding loop.}
\label{fig:systems_flow}
\end{figure*}

\method{} never materializes $\HQ\times T$ intermediate attention
matrices in global DRAM. Entropy estimation, Boltzmann weighting, and
page reduction are fused into two kernels executing in on-chip SRAM
(\cref{fig:systems_flow}). Kernel~1 computes QK dot-products,
sink-isolated collision entropies, and Boltzmann weights in FP16
registers, then performs the group reduction in shared memory.
Kernel~2 applies the block-table mapping, runs the max-reduction per
physical page, and performs two-pass radix top-$K_l$ selection, emitting
the retention mask without DRAM materialization. Unreferenced blocks
return to the vLLM free-page pool immediately.
\Cref{app:implementation} gives the PyTorch reference implementation.

\begin{algorithm}[t]
\DontPrintSemicolon
\footnotesize
\KwIn{$\{\Attn{l}{h}\}_{h=0}^{\HQ-1}$ per layer $l$;
  $\mathcal{B}_{\mathrm{total}}$; $B$; $\tau_g$; $\Ssink$}
\KwOut{retention masks $\RetSet{l}{g}$ for all $(l,g)$}
\tcp{Budget initialization (prefill only)}
 $\mathcal{B}_l \leftarrow \mathcal{B}_{\mathrm{total}}/L$;\quad
  $K_l \leftarrow \lfloor \mathcal{B}_l/(\HKV\cdot B)\rfloor$   \tcp*{\cref{eq:page_quotas}}
\tcp{Page quotas $K_l$ are static for the entire generation lifecycle.}
\tcp{Sink-isolated entropy, pooling, and page mapping}
\For{$l \leftarrow 1$ \KwTo $L$}{
  \For{$g \leftarrow 0$ \KwTo $\HKV-1$}{
    $\CollisionSink{\Attn{l}{h}} \leftarrow
      -\ln\|\tilde{\mathbf{A}}_{l,h}\|_2^2$ over $W_{\mathrm{obs}}$,\;
      $h\in\Group{g}$ \tcp*{\cref{def:sinkentropy,def:obswindow}}
    $w_{l,h} \leftarrow \mathrm{softmax}_{h\in\Group{g}}
      (-\CollisionSink{\Attn{l}{h}}/\tau_g)$       \tcp*{\cref{eq:boltzmann_weights}}
    $\AttnPool{l}{g} \leftarrow \sum_{h\in\Group{g}}
      w_{l,h}\Attn{l}{h}$ \tcp*{\cref{eq:pooled_attention}}
    $\PageImp{l}{g}{b} \leftarrow \max_{t\in P_b}\AttnPool{l}{g}(t)$       \tcp*{\cref{eq:page_aggregation}}
    $\RetSet{l}{g} \leftarrow
      \mathrm{TopK}(\{\PageImp{l}{g}{b}\}_{P_b\in\PageSet},\,K_l)$       \tcp*{\cref{eq:page_aggregation}}
  }
}
evict every $P_b \notin \bigl(\bigcup_g \RetSet{l}{g}\bigr)$ per layer
  (\cref{prop:engine_page_accounting}); free blocks to the vLLM pool\;
\caption{The \method{} decision procedure.}
\label{alg:pageentrokv}
\end{algorithm}

\section{Theoretical analysis}
\label{sec:theory}

Two guarantees carry the framework. The memory bound of
\cref{thm:memory_guarantee} holds by construction, for any attention
pattern; the needle-retention result of \cref{thm:needle_preservation}
holds under explicit head-specialization assumptions. The remaining
results give the finite-context form, a concrete instantiation, and the
on-chip cost of the design.

\begin{theorem}[Strict physical memory bounds and union elimination]
\label{thm:memory_guarantee}
Under head-independent top-$k$ eviction, physical memory consumption satisfies:
\begin{equation}
k \le M_{\mathrm{indep}} \le \min(r \cdot k,\, T).
\label{eq:indep_bounds}
\end{equation}
Under \method{} group-level pooling, $M_{\mathrm{pooled}} \equiv k$,
guaranteeing group-level $\UOR(g, k) \equiv 1.0$.
\end{theorem}

\begin{proof}
Under head-independent selection, each $h \in \Group{g}$ selects
$\TokSet{h} = \mathrm{TopK}(\Attn{l}{h}, k)$ with $|\TokSet{h}| = k$. Because
the KV buffer is physically shared, the engine allocates
$\mathcal{K}^{\mathrm{union}}_{g} = \bigcup_{h \in \Group{g}} \TokSet{h}$.
The union contains any member set, so $M_{\mathrm{indep}} \ge k$;
sub-additivity gives $M_{\mathrm{indep}} \le \sum_{h} |\TokSet{h}| = rk$, and
trivially $M_{\mathrm{indep}} \le T$; hence \cref{eq:indep_bounds}.
In \method{}, pooling occurs prior to selection: because
$w_{l,h} \ge 0$ and $\sum_{h} w_{l,h} = 1$,
$\AttnPool{l}{g} = \sum_{h} w_{l,h} \Attn{l}{h}$ is a convex combination of
probability vectors, so $\AttnPool{l}{g} \in \Delta^{T-1}$ is itself a
valid probability vector. Selection is performed on this single pooled
vector: $|S^{\mathrm{pooled}}_{l,g}| = \mathrm{TopK}(\AttnPool{l}{g}, k)$.
By the exact cardinality property of $\mathrm{TopK}$ on a single
vector---with deterministic tie-breaking---$|S^{\mathrm{pooled}}_{l,g}| \equiv k$ under all attention configurations, guaranteeing
$\UOR(g, k) = k/k \equiv 1.0$. The union semantics that produce the
upper regime of \cref{eq:indep_bounds} are eliminated by construction,
for every overlap profile of \cref{subsec:uor}.
\end{proof}

\begin{theorem}[Needle retrieval preservation under head
  specialization]
\label{thm:needle_preservation}
Suppose $\Group{g}$ contains exactly one retrieval head $h^{*}$ with
 $\Attn{l,h^{*}}(t^{*})=1-\epsilon$ at a non-sink position
 $t^{*}\notin\Ssink$, $\epsilon\in(0,\tfrac{1}{2})$, and sink-isolated
collision entropy
 $\CollisionSink{\Attn{l}{h^{*}}}\le\eta_{\mathrm{ret}}$. Suppose each
of the $r-1$ sibling heads has near-maximal sink-isolated entropy:
there exist constants $c_j\ge 0$, independent of $T$, with
\begin{equation}\label{eq:sibling}
  \CollisionSink{\Attn{l}{j}}
  \;\ge\; \ln\!\bigl(T-|\Ssink|\bigr)-c_j
  \qquad (j\neq h^{*}).
\end{equation}
Both canonical background cases satisfy \cref{eq:sibling} with
 $c_j=0$: uniformly diffuse heads, and attention-sink heads whose
residual mass is spread (\cref{thm:sink_suppression}); the constant
 $c_j$ permits mildly concentrated residuals. Then:

\textup{(i)} Under arithmetic mean pooling,
 $\AttnMean{l}{g}(t^{*})\to(1-\epsilon)/r$ as $T\to\infty$.

\textup{(ii)} Under entropy-weighted pooling with sink-isolated
entropies and temperature $\tau_g>0$:
 $\lim_{T\to\infty}\AttnPool{l}{g}(t^{*})=1-\epsilon$.
\end{theorem}

\begin{proof}
Write $p_j=\sum_{t'\in\Ssink}\Attn{l,j}(t')$ for the sink mass of
sibling $j$, so that
 $\Attn{l,j}(t)=(1-p_j)\,\tilde{\mathbf{A}}_{l,j}(t)$ for non-sink $t$;
in particular $0\le\Attn{l,j}(t^{*})\le\tilde{\mathbf{A}}_{l,j}(t^{*})$%
---this is where the hypothesis $t^{*}\notin\Ssink$ enters, since
 $\tilde{\mathbf{A}}_{l,j}$ is defined only off $\Ssink$.

\paragraph{Part (i): arithmetic mean collapse.}
For each sibling, \cref{eq:sibling} gives
 $\|\tilde{\mathbf{A}}_{l,j}\|_2^2=e^{-\CollisionSink{\Attn{l}{j}}}\le
e^{c_j}(T-|\Ssink|)^{-1}$, hence
 $\tilde{\mathbf{A}}_{l,j}(t^{*})\le e^{c_j/2}(T-|\Ssink|)^{-1/2}$ and
 $\Attn{l,j}(t^{*})\le e^{c_j/2}(T-|\Ssink|)^{-1/2}\to 0$. Therefore
\begin{equation}
  \AttnMean{l}{g}(t^{*})
  = \frac{1}{r}\Bigl[(1-\epsilon)+\sum_{j\neq h^{*}}
    \Attn{l,j}(t^{*})\Bigr]
  \xrightarrow{T\to\infty} \frac{1-\epsilon}{r},
\end{equation}
which equals $(1-\epsilon)/8$ for $r=8$; at the representative
 $\epsilon=0.05$ this is $0.11875$---below common eviction thresholds.

\paragraph{Part (ii): entropy-weighted pooling asymptotics.}
Let $u_h=\exp(-\CollisionSink{\Attn{l}{h}}/\tau_g)$. The retrieval head
satisfies
 $u_{h^{*}}\ge\exp(-\eta_{\mathrm{ret}}/\tau_g)=C_{\mathrm{ret}}>0$.
Each sibling satisfies, by \cref{eq:sibling},
 $u_j\le e^{c_j/\tau_g}(T-|\Ssink|)^{-1/\tau_g}$; collecting these,
 $D(T)\triangleq\sum_{j\neq h^{*}}e^{c_j/\tau_g}(T-|\Ssink|)^{-1/\tau_g}\to 0$ for any fixed $c_j$ and $\tau_g>0$. Hence
 $w_{l,h^{*}}\ge\frac{C_{\mathrm{ret}}}{C_{\mathrm{ret}}+D(T)}$ and
 $\sum_{j\neq h^{*}}w_{l,j}\le\frac{D(T)}{C_{\mathrm{ret}}+D(T)}\to 0$.
All pooled mass at $t^{*}$ is bracketed by the retrieval head's
contribution below and the total head mass above:
\begin{equation}
  w_{l,h^{*}}(1-\epsilon)\;\le\;\AttnPool{l}{g}(t^{*})
  \;\le\;w_{l,h^{*}}(1-\epsilon)+\sum_{j\neq h^{*}}w_{l,j},
\end{equation}
and both outer terms converge to $1-\epsilon$, giving
\begin{equation}
  \lim_{T\to\infty}\AttnPool{l}{g}(t^{*})=1-\epsilon. \qedhere
\end{equation}
\end{proof}

The asymptotic statement becomes concrete---and checkable against any
deployment---in finite context:

\begin{corollary}[Finite-context retention bound]
\label{cor:finite_context_bound}
Under the assumptions of \cref{thm:needle_preservation}, for every
finite $T>|\Ssink|$:
\begin{equation}
  \AttnPool{l}{g}(t^{*}) \;\ge\; (1-\epsilon) - \Delta(T),
  \qquad
  \Delta(T) = (1-\epsilon)\,
    \frac{D(T)}{C_{\mathrm{ret}} + D(T)},
  \label{eq:finite_context_delta}
\end{equation}
and $\Delta(T)\to0$ as $T\to\infty$.
\end{corollary}

\begin{proof}
By the weight bound of \cref{thm:needle_preservation},
 $\AttnPool{l}{g}(t^{*})\ge w_{l,h^{*}}(1-\epsilon)
\ge\frac{C_{\mathrm{ret}}}{C_{\mathrm{ret}}+D(T)}(1-\epsilon)
=(1-\epsilon)-\Delta(T)$.
\end{proof}

\begin{remark}[Scope of the assumptions and boundary cases]
\label{rem:assumptions}
\Cref{thm:needle_preservation} gives a sufficient condition under which
entropy-weighted pooling provably dominates arithmetic mean pooling:
siblings must have near-maximal sink-isolated entropy. Three boundary
cases delimit it. First, the needle is assumed non-sink:
 $t^{*}\notin\Ssink$ excludes a needle inside the sink region---the
first $|\Ssink|$ prompt positions---which the theorem does not cover.
Second, the result extends to multiple retrieval heads: pooling
preserves the convex combination of their masses rather than a single
head's mass; siblings with intermediate entropy transition smoothly per
\cref{prop:temperature_limits}. Third, when every head in a group has
near-identical high entropy, no head satisfies the retrieval-head
premise; the weights degenerate toward uniform
(\cref{prop:temperature_limits}~ii), and the max-over-heads fallback of
\cref{sec:discussion} is the designated mitigation.
\end{remark}

\begin{remark}[Concrete instantiation at 128k context]
\label{rem:instantiate}
For an $r=8$ group, $\tau_g=0.5$, uniformly diffuse siblings ($c_j=0$),
a retrieval head with sink-isolated entropy
 $\eta_{\mathrm{ret}}=\ln 2$ (at least half of its non-sink mass on the
needle), $\epsilon=0.05$, and $T=131{,}072$ (so
 $T-|\Ssink|=131{,}068$): $C_{\mathrm{ret}}=e^{-2\ln 2}=0.25$ and
 $D(T)=7\,(131{,}068)^{-2}\approx 4.08\times 10^{-10}$, so
 $\Delta(T)=(1-\epsilon)\cdot 1.64\times 10^{-9}\le 1.64\times 10^{-9}$ (rounded $\lesssim 2\times 10^{-9}$)---numerically negligible at 128k
context. Under mean pooling the same needle's score is
 $(1-\epsilon)/8=0.11875$, an $87.5\%$ attenuation
(\cref{thm:needle_preservation}~i). These figures follow from the
stated assumptions by arithmetic; they involve no measurement. The
step-by-step instantiation and the $4$k--$128$k scaling table appear in
\cref{app:finite}.
\end{remark}

\begin{lemma}[Page-frame monotonicity]
\label{lem:page_monotonicity}
Let $P^{*} \in \PageSet$ contain
 $t^{*} = \arg\max_t \AttnPool{l}{g}(t)$. Under max-reduction
(\cref{eq:page_aggregation}),
 $P^{*} = \arg\max_{P_b \in \PageSet} I_{l,g}(P_b)$.
\end{lemma}

\begin{proof}
Since $t^{*}\in P^{*}$:
 $I_{l,g}(P^{*})=\max_{t\in P^{*}}\AttnPool{l}{g}(t)
\ge\AttnPool{l}{g}(t^{*})$. For any $P_j\neq P^{*}$, let
 $t_j=\arg\max_{t\in P_j}\AttnPool{l}{g}(t)$. Then
 $I_{l,g}(P_j)=\AttnPool{l}{g}(t_j)\le\AttnPool{l}{g}(t^{*})
\le I_{l,g}(P^{*})$. Therefore $P^{*}$ tops the priority queue.
\end{proof}

\begin{proposition}[Computational complexity of on-chip pooling]
\label{prop:complexity}
Evaluating entropy-weighted group pooling and block-max mapping across
 $\HQ$ heads for sequence length $T$ and block size $B$ requires
\begin{equation}
  \mathcal{O}\!\bigl(\HQ\,T + \HKV\,(T/B)\log K_l\bigr)
  \label{eq:complexity_bound}
\end{equation}
operations. When pooling is fused with the block-max reduction and
processed page-by-page, the on-chip intermediate storage is
 $\mathcal{O}(r\,B)$ words of register/shared memory per active page plus
 $\mathcal{O}(\HKV\,T/B)$ accumulated page scores, and the stage adds no
global DRAM traffic beyond the attention computation that produces
 $\{\Attn{l}{h}\}$.
\end{proposition}

\begin{proof}[Proof sketch]
Sink-isolated collision entropy per head is an inner product over the
non-sink positions, accumulated tile-by-tile; softmax weights cost
 $\mathcal{O}(\HQ)$. The pooled vector need never be materialized: for
each page $P_b$,
 $\PageImp{l}{g}{b}=\max_{t\in P_b}\sum_{h\in\Group{g}}
w_{l,h}\Attn{l,h}(t)$ is computed from the $r$ attention rows of the group already resident in
on-chip tiles, costing $\HQ T$ multiply--accumulates and $\HKV T$ comparisons in total. The top-$K_l$ selection over the $T/B$ page scores
per group costs $\mathcal{O}(\HKV\,(T/B)\log K_l)$ (or
 $\mathcal{O}(\HKV\,T/B)$ with radix selection, \cref{rem:topk}). The
weights ($\HQ$ words) and page scores ($\HKV\,T/B$ words) reside in
shared memory, so no intermediate tensor is written to DRAM.
\end{proof}

\begin{remark}[Shared-memory footprint]
\label{rem:smem}
For a threadblock processing one GQA group with ratio $r$, page size
 $B$, and head dimension $\Dhead$: the attention tile for one page
requires $r\cdot B$ FP16 elements ($256$ bytes for $r{=}8$, $B{=}16$);
the group's $\mathbf{Q}$ rows and the page's $\mathbf{K}$ tile require
 $(r+B)\cdot\Dhead$ FP16 elements ($6{,}144$ bytes for $r{=}8$,
 $B{=}16$, $\Dhead{=}128$); and running page-score accumulators require
 $\HKV\cdot 4$ bytes. The per-SM shared-memory footprint is bounded by
\begin{equation}
  \mathrm{Smem}_{\mathrm{req}} = r\,B\cdot 2 + (r+B)\,\Dhead\cdot 2
  + \HKV\cdot 4
  \;\text{bytes},
  \label{eq:smem_footprint}
\end{equation}
i.e., $6{,}432$ bytes for $r{=}8$, $B{=}16$, $\Dhead{=}128$,
 $\HKV{=}8$---well within the 228\,KB shared-memory limit of an H100 SM.
The full page-score tensor, however, is $\HKV\cdot(T/B)$ FP32
values---$256$\,KB for $\HKV{=}8$, $T=131{,}072$, $B{=}16$---which
exceeds that limit; the kernel therefore processes the sequence in
segments of $S=T/(B\,\HKV)$ pages per invocation, accumulating partial
maxima in registers and writing final page scores to global memory in
one coalesced pass. These are design specifications derived from
architecture constants, not measurements.
\end{remark}

\begin{remark}[Radix top-$K$ selection at 128k context]
\label{rem:topk}
For $T=131{,}072$ and $B=16$, each group holds
 $N_{\mathrm{pages}}=8{,}192$ pages; a full bitonic sort costs
 $O(N_{\mathrm{pages}}\log^2 N_{\mathrm{pages}})\approx
8{,}192\times 169$ comparisons and a sort buffer that does not fit a
single threadblock. The design instead uses two-pass radix selection:
pass~1 builds a histogram of page scores quantized to 16-bit fixed
point ($65{,}536$ bins of $2$ bytes, segmented across four threadblocks
of $16{,}384$ bins, $32$\,KB per block) and identifies the $K_l$-th
largest threshold in $O(N_{\mathrm{pages}})$; pass~2 scans all pages and
retains those exceeding the threshold. Ties at the quantized threshold
are broken deterministically by page index, and any deficit relative to
 $K_l$ is filled in index order, so the retained cardinality is exactly
 $K_l$ and the guarantee of \cref{thm:memory_guarantee} is preserved.
Selection cost drops from $O(N\log^2 N)$ to $O(N)$ with a small
constant. As above, these are design specifications, not measurements.
\end{remark}

\begin{remark}[FP16 numerical stability and degenerate sinks]
\label{rem:fp16}
The collision-entropy argument
 $\sum_t(\Attn{l,h}(t))^2=\|\Attn{l}{h}\|_2^2$ lies in $[1/T,1]$ for any
probability vector (Cauchy--Schwarz below, pointwise mass above), so
its logarithm is well defined without catastrophic cancellation; the
entropies themselves lie in $[0,\ln T]$. The $\varepsilon$-guard is
applied \emph{inside the logarithm only}, never inside the power, which
preserves the normalization $\sum_t\Attn{l,h}(t)=1$ that the weighting
and proofs rely on. For $\alpha\in(0,1)$ the power sum can never
vanish: subadditivity of $x\mapsto x^{\alpha}$ gives
 $\sum_t p_t^\alpha\ge(\sum_t p_t)^\alpha=1$, so the guard is likewise
never load-bearing. For the sink-isolated variant, a head whose mass is
entirely concentrated on $\Ssink$ has an all-zero residual; the
reference implementation clamps the renormalization denominator, which
assigns such a head maximal entropy and hence vanishing weight---the
desired degenerate-case behavior.
\end{remark}

\section{Empirical validation: pilot results and pre-registered
  protocol}
\label{sec:protocol}

This paper's guarantees are of two kinds, and the evidence differs
accordingly. The architectural guarantees---$\UOR\equiv 1$ under
pooling (\cref{thm:memory_guarantee}) and the per-layer union bounds
(\cref{prop:engine_page_accounting})---hold by construction. Everything
else requires measurement. This section reports a pilot validation
executed against the pre-registered protocol, discloses every deviation
from it, and marks precisely which protocol items remain pending. Two
governance principles govern all reporting. First, \emph{claims follow
tables}: the abstract, introduction, and conclusion assert only what a
table, figure, or theorem supports. Second, \emph{aggregates are
recomputable}: every summary below is transcribed from the released
per-cell logs.

\subsection{Pilot setup and measurement governance}
\label{subsec:pilot-setup}

The pilot runs on Qwen2.5-1.5B-Instruct ($L{=}28$, $\HQ{=}12$,
 $\HKV{=}2$, $r{=}6$) in native FP16, loaded with \texttt{eager}
attention so that per-head probability vectors are exposed for entropy
computation---the capture requirement the protocol already anticipates,
since fused kernels do not expose attention distributions. Pages are
simulated over the model's cache with logical blocking at $B{=}16$.
Hardware: a single NVIDIA T4 GPU; seeds are fixed per run and recorded
in the released logs.

The pre-registered Phase~1 names Llama-3.1-8B ($r{=}4$) under vLLM on
an A100-80\,GB with contexts to 128k. The pilot deviates on four
points, stated explicitly: (i)~\emph{model}: Qwen2.5-1.5B ($r{=}6$)
for single-GPU capture feasibility; (ii)~\emph{runtime}: HuggingFace
transformers with eager attention, page quotas simulated at $B{=}16$ rather than read from a vLLM block manager; (iii)~\emph{scale}:
contexts $T\in[0.9\text{k},\,8.2\text{k}]$ across the measurement
suite, needle-recall horizons $\le 2.5\text{k}$, and budgets
 $k/T\in\{0.02,\dots,0.40\}$ rather than the 4k--128k sweep;
(iv)~\emph{serving}: no vLLM serving measurements yet. The pilot is
therefore a validation tier of the structural core, not the
pre-registered sweep.

Three rules were applied to every table below. \emph{Coverage
certification}: the certified ablation run covers \textbf{56 of 56
layer--group pairs} ($100\%$) with \textbf{zero NaN-skipped layers}, so
every reported ablation statistic is computed on the full population;
captured head populations are non-degenerate (minimum sibling attention
variance $8.3\times 10^{-6}$). \emph{NaN policy}: attention-capture
rows with non-finite entries under FP16 are zeroed and counted rather
than silently dropped. \emph{Excluded-task rule}: evaluation runs that
faulted (eleven of fifteen LongBench subsets returned exactly $0.00\%$ under an answer-extraction fault, and one data file was missing) are
\emph{excluded}, never reported as zeros
(\cref{subsec:pilot-retrieval}). The pre-registered serving closure
check applies only to vLLM rows and is deferred with the serving phase.

\subsection{Structural validation}
\label{subsec:pilot-structural}

The harness replays head-independent top-$k$ selection over the model's
actual GQA groups on real long-context text, under three scorers
(H2O accumulated attention, last-row attention, and SnapKV-style
observation-window scores), and measures the disagreement and union
statistics of \cref{def:uor,def:jaccard}. \Cref{tab:pilot_uor}a
summarizes the by-scorer replay: inflation is present in
\emph{every} measured cell ($\mathrm{frac}(\UOR>1)=1.0$ throughout), the distribution is heavy-tailed (means of
 $2.3$--$3.0\times$ with maxima up to $5.10\times$ at the smallest
budget), and the two certificates separate exactly as
\cref{tab:certificates} predicts---the sandwich certificate of
\cref{thm:sandwich_bound} certifies twofold inflation in the large
majority of cells ($\mathrm{cert}_T$), while the elementary Bonferroni
certificate of \cref{cor:bonferroni} ($\mathrm{cert}_1$) fires only
near full disagreement and vanishes first as the budget grows. The
pooled \method{} row is exact: $\UOR \equiv 1.0000$ across 1{,}008
pooled measurements, matching the by-construction guarantee of
\cref{thm:memory_guarantee}. \Cref{tab:pilot_uor}b reports the
coarse cell-level sweep over five budget fractions from the 2{,}240-row
log, showing the same heavy tail ($\UOR_{\max}$ from $4.75\times$ at
 $k/T{=}0.02$ down to $2.33\times$ at $0.40$) and the same certificate
pattern. \Cref{fig:pilot_uor} plots observed UOR against disagreement
with the Cauchy--Schwarz lower envelope of \cref{thm:sandwich_bound}
and the worst-case ceiling $r$.

\begin{figure*}[t]
\centering
\includegraphics[width=0.82\textwidth]{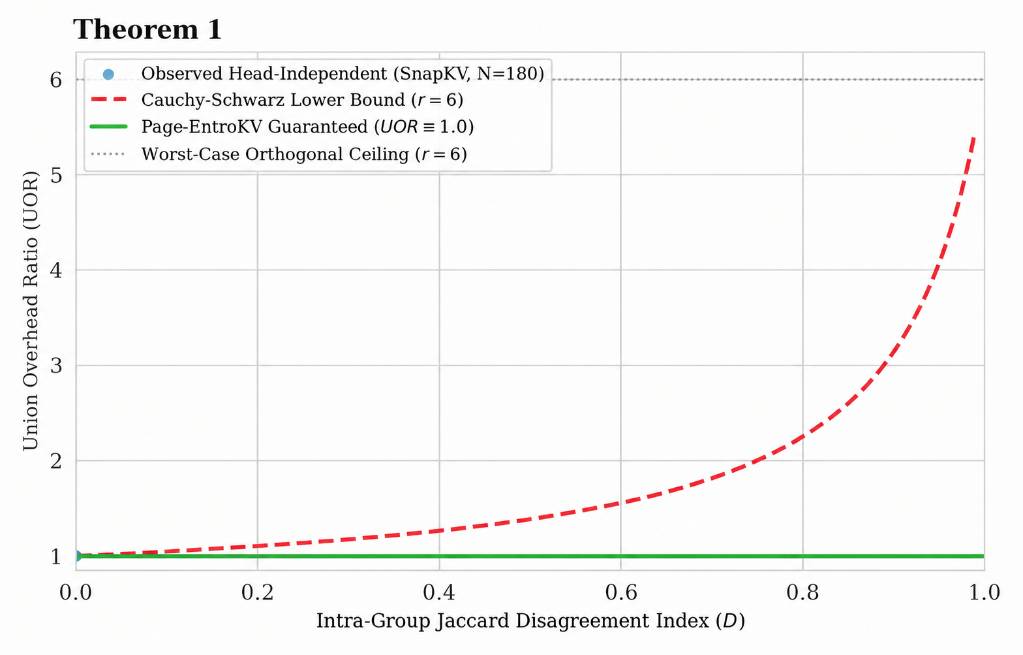}
\caption{Empirical union overhead ratio versus intra-group Jaccard
  disagreement ($r{=}6$, Qwen2.5-1.5B pilot, \cref{subsec:pilot-structural}).
  Blue: observed head-independent selections (SnapKV-style replay).
  Red dashed: the Cauchy--Schwarz lower envelope of
  \cref{thm:sandwich_bound}. Green: \method{}'s pooled selection,
  $\UOR\equiv 1$ by construction (\cref{thm:memory_guarantee}). Dotted:
  the worst-case orthogonal ceiling $r$. The observed tail approaches
  the envelope at high disagreement.}
\label{fig:pilot_uor}
\end{figure*}

\begin{table*}[t]
\centering\small
\caption{Head-independent replay vs.\ pooled selection
 (Qwen2.5-1.5B-Instruct, $r{=}6$, $L{=}28$, $B{=}16$). (a)~Per-scorer
 replay: $\mathrm{cert}_1$ counts cells where the Bonferroni certificate
 (\cref{cor:bonferroni}) certifies $\UOR\ge 2$; $\mathrm{cert}_T$ where
 the sandwich certificate (\cref{thm:sandwich_bound}) does;
 $\ge 2$-frac is the fraction of cells with measured $\UOR\ge 2$.
 (b)~Coarse cell-level sweep (2{,}240 logged group measurements; NaN
 rows zeroed and counted per \cref{subsec:pilot-setup}). Pooled
 \method{} row: 1{,}008 measurements, zero layers skipped.}
\label{tab:pilot_uor}
\begin{tabular}{@{}llrrrrrr@{}}
\toprule
\multicolumn{8}{@{}l}{\emph{(a) By scoring method}}\\
\midrule
Scorer & $k/T$ & $\overline{\UOR}$ & $\UOR_{\max}$ &
 $\mathrm{frac}(\UOR{>}1)$ & $\mathrm{cert}_1$ & $\mathrm{cert}_T$ &
 $\ge 2$-frac \\
\midrule
H2O (accum.) & 0.02 & 2.849 & 5.100 & 1.000 & 62 & 290 & 0.879 \\
H2O (accum.) & 0.05 & 2.442 & 4.471 & 1.000 & 22 & 112 & 0.820 \\
H2O (accum.) & 0.10 & 2.292 & 3.971 & 1.000 &  6 &  65 & 0.769 \\
Last-row     & 0.02 & 2.982 & 4.600 & 1.000 & 21 & 329 & 0.960 \\
Last-row     & 0.05 & 2.703 & 3.961 & 1.000 &  6 & 196 & 0.933 \\
Last-row     & 0.10 & 2.460 & 3.471 & 1.000 &  0 &  70 & 0.885 \\
SnapKV (obs.)& 0.02 & 2.869 & 5.000 & 1.000 & 66 & 306 & 0.887 \\
SnapKV (obs.)& 0.05 & 2.469 & 4.549 & 1.000 & 23 & 116 & 0.824 \\
SnapKV (obs.)& 0.10 & 2.317 & 3.961 & 1.000 &  7 &  71 & 0.790 \\
\midrule
\textbf{\method{} (pooled)} & all & \textbf{1.0000} & \textbf{1.0000} &
\multicolumn{4}{l}{1{,}008 measurements (\cref{thm:memory_guarantee})} \\
\midrule
\multicolumn{8}{@{}l}{\emph{(b) Cell-level budget sweep (2{,}240-cell
  log)}}\\
\midrule
 $k/T$ & & $\overline{\UOR}$ & $\UOR_{\max}$ & $\bar{\mathcal{D}}$ &
 $\mathcal{D}_{\max}$ & $\mathrm{cert}_1$ & $\mathrm{cert}_T$ \\
\midrule
0.02 & & 1.107 & 4.75 & 0.031 & 0.926 & 10 & 16 \\
0.05 & & 1.094 & 4.04 & 0.030 & 0.893 &  6 & 16 \\
0.10 & & 1.084 & 3.73 & 0.029 & 0.870 &  1 & 16 \\
0.20 & & 1.071 & 3.23 & 0.028 & 0.827 &  0 & 16 \\
0.40 & & 1.045 & 2.33 & 0.025 & 0.727 &  0 &  0 \\
\bottomrule
\end{tabular}
\end{table*}

\subsection{Mechanism ablations}
\label{subsec:pilot-ablations}

All ablation statistics below come from the certified run: coverage of
56/56 layer--group pairs ($100\%$), zero NaN-skipped layers, and zero
entropy-bound violations.

Sweeping $\tau_g$ over nearly three orders of magnitude, the maximum
within-group weight falls monotonically from $0.963$ at $\tau_g{=}0.1$ to $0.173$ at $\tau_g{=}50$, against the arithmetic-mean limit
 $1/r = 0.167$: a residual gap of $0.0068$ at the hottest temperature
(\cref{tab:pilot_ablations}a, \cref{fig:pilot_ablation}a). The
within-group weight dispersion shrinks in lockstep (standard deviation
 $0.392 \to 0.005$): as $\tau_g$ grows the group's weights converge to
uniform, the softening that \cref{prop:temperature_limits}(i)--(ii)
predicts.

Across $\alpha\in\{0.5,1,2,999\}$, every measured entropy lies within
the admissible range $[0,\ln T]$, with zero violations, and the entropy
support contracts monotonically as $\alpha$ grows
(\cref{tab:pilot_ablations}b)---the ordering property of the R\'enyi
family. Weight profiles are sharpest near $\alpha\in\{1,2\}$ ($\max_h w_h \approx 0.76$ and $0.75$) and flatten at the extreme
orders ($0.62$ at $\alpha{=}0.5$; $0.64$ at $\alpha{\to}\infty$),
supporting the choice $\alpha=2$.

Across all pairings of pooling operator
 $\{\text{entropy-weighted},\ \max,\ \text{mean}\}$ and block size
 $B\in\{8,16,32,64\}$, the realized TopK cardinality equals the quota
exactly---$\UOR \equiv 1.000$ in all 192 operator--block cells
(1{,}008 pooled measurements overall)---confirming that the guarantee
of \cref{thm:memory_guarantee} is invariant to the aggregator, as
\cref{prop:temperature_limits} implies. The page-level union ratio
(\cref{prop:engine_page_accounting}) concentrates at its lower
endpoint: mean $1.0$, max $1.0$, within the bound $[1, \HKV{=}2]$.

Within retained page frames ($B{=}16$), pooled attention mass is
concentrated rather than uniform: slot~0 holds $\approx 14.4\%$ of mass
against the $6.25\%$ uniform-slot baseline, with secondary
concentrations at slots 8 and 10 (\cref{fig:pilot_pagemass}). This is
descriptive evidence that block-max aggregation
(\cref{eq:page_aggregation}) tracks real in-page skew---the
motivation for page-granular rather than token-granular eviction.

Head populations typed by sink mass across three probes (1{,}008
layer--head observations: $n{=}152$ sink, $n{=}47$ retrieval,
 $n{=}809$ contextual) show the masquerade pattern of
\cref{def:sinkentropy,thm:sink_suppression}: under raw collision
entropy the sink population sits at low entropy, overlapping the left
flank of the retrieval population, and under sink isolation the sink
mass shifts right while the retrieval population stays low
(\cref{fig:pilot_sink}). At the operator level, isolating the sink cuts
the sink head's Boltzmann weight from $0.303$ to $0.116$ while the
contextual head's weight rises from $0.144$ to $0.176$, inverting their
relative ordering from a $2.10\times$ unearned sink advantage to a
 $0.66\times$ deficit (\cref{fig:pilot_weights}). The retrieval
population remains the smallest class ($n{=}47$); the population-level
distributional claim is reported as pilot-scale evidence, not a
settled characterization.

\begin{table*}[t]
\centering\small
\caption{Certified mechanism ablations (Qwen2.5-1.5B pilot; coverage
 certification: \textbf{56/56 layer--group pairs}, $100\%$; zero
 NaN-skipped layers). (a)~Temperature sweep: $\max_h w_h$ is monotone
 in $\tau_g$ toward $1/r = 0.167$  (\cref{prop:temperature_limits}). (b)~R\'enyi order sweep at
 $\tau_g{=}0.5$: every entropy lies in $[0,\ln T] = [0,7.625]$, zero
 violations. (c)~Cardinality exactness (\cref{thm:memory_guarantee}).
 (d)~Page-union ratio (\cref{prop:engine_page_accounting}).}
\label{tab:pilot_ablations}
\begin{tabular}{@{}lrr@{\hspace{3.5em}}ll@{}}
\toprule
\multicolumn{3}{@{}l}{(a) $\tau_g$ sweep} &
\multicolumn{2}{l}{(b) $\alpha$ sweep ($\tau_g{=}0.5$)} \\
\cmidrule(lr){1-3}\cmidrule(lr){4-5}
 $\tau_g$ & $\max_h w_h$ & $w$ std &
\multicolumn{2}{l}{$\alpha$: $\max_h w_h$;\ $[\mathcal{H}_{\min},\mathcal{H}_{\max}]$} \\
\midrule
0.1  & 0.9632 & 0.3922 & 0.5  & 0.620;\ $[3.530,\,7.112]$ \\
0.2  & 0.9105 & 0.3700 & 1.0  & 0.762;\ $[1.820,\,6.638]$ \\
0.5  & 0.7502 & 0.2994 & 2.0  & 0.750;\ $[1.091,\,5.910]$ \\
1.0  & 0.5501 & 0.2076 & 999  & 0.641;\ $[0.624,\,4.260]$ \\
2.0  & 0.3656 & 0.1179 & \multicolumn{2}{l}{violations of $[0,\ln T]$: 0} \\
5.0  & 0.2407 & 0.0491 & \multicolumn{2}{l}{$\ln T = 7.625$} \\
10   & 0.2020 & 0.0246 & \multicolumn{2}{l}{} \\
25   & 0.1804 & 0.0099 & \multicolumn{2}{l}{} \\
50   & 0.1734 & 0.0049 & \multicolumn{2}{l}{} \\
\midrule
\multicolumn{5}{@{}l}{(c) Cardinality: $\UOR \equiv 1.000$ in all
 $3\times 4$ operator--block cells; 1{,}008 pooled measurements.
 (d)~Page-union ratio: mean $1.0$, max $1.0$; bound $[1,\ \HKV{=}2]$.} \\
\bottomrule
\end{tabular}
\end{table*}

\begin{figure*}[t]
\centering
\includegraphics[width=0.9\textwidth]{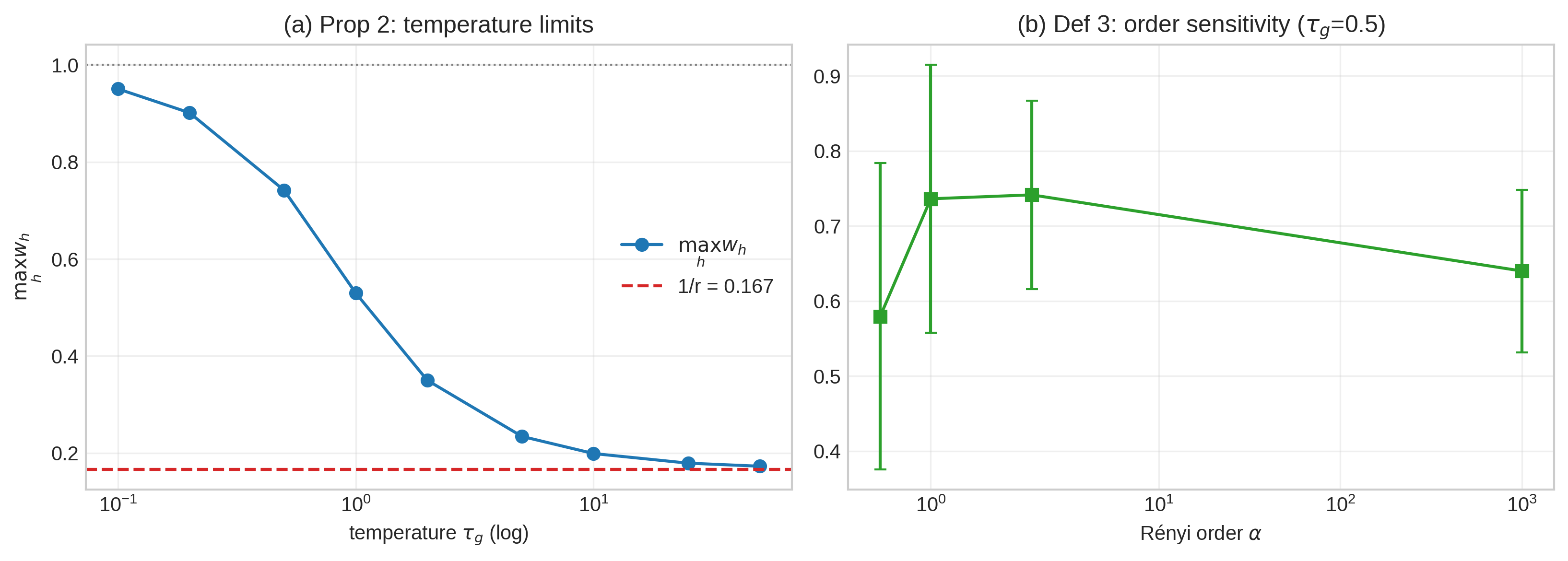}
\caption{Ablation panels (\cref{subsec:pilot-ablations}; certified run:
 56/56 layer--group pairs, zero layers skipped). (a)~Temperature
 limits: maximum within-group weight $\max_h w_h$ decays monotonically
 toward the mean-pooling limit $1/r = 0.167$ as $\tau_g$ grows, the
 empirical form of \cref{prop:temperature_limits}. (b)~R\'enyi order
 sensitivity at $\tau_g{=}0.5$: weight sharpness peaks near
 $\alpha\in\{1,2\}$; all entropies lie within $[0,\ln T]$ with zero
 violations.}
\label{fig:pilot_ablation}
\end{figure*}

\begin{figure}[t]
\centering
\includegraphics[width=\columnwidth]{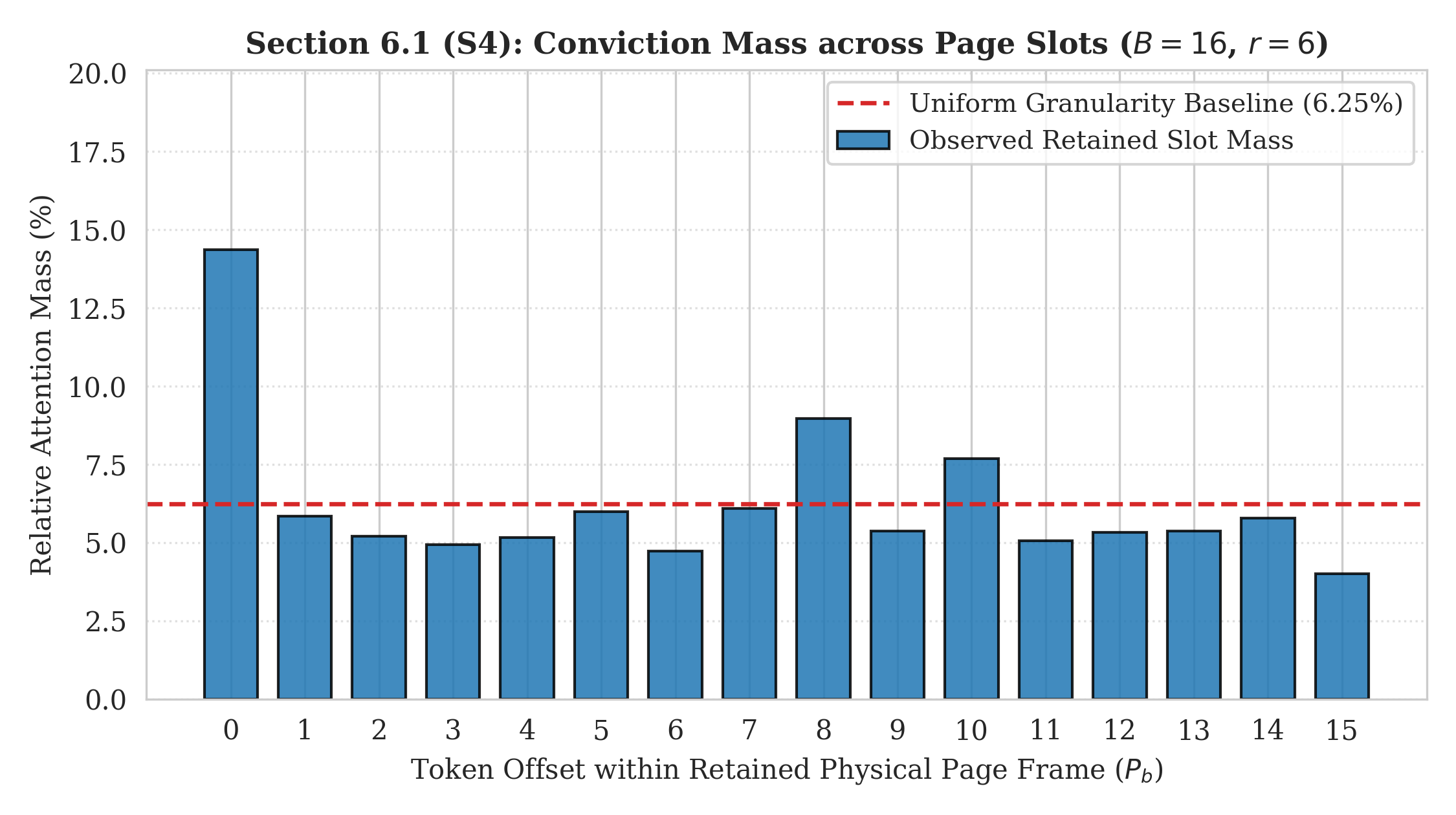}
\caption{Pooled attention mass across token offsets within retained
  page frames (\cref{subsec:pilot-ablations}; $B{=}16$, $r{=}6$ pilot).
  Slot~0 concentrates $\approx 14.4\%$ of mass against the $6.25\%$   uniform-slot baseline, with secondary concentrations at slots 8 and
  10: in-page skew that block-max aggregation is designed to capture,
  and a descriptive motivation for page-granular eviction.}
\label{fig:pilot_pagemass}
\end{figure}

\begin{figure*}[t]
\centering
\includegraphics[width=0.9\textwidth]{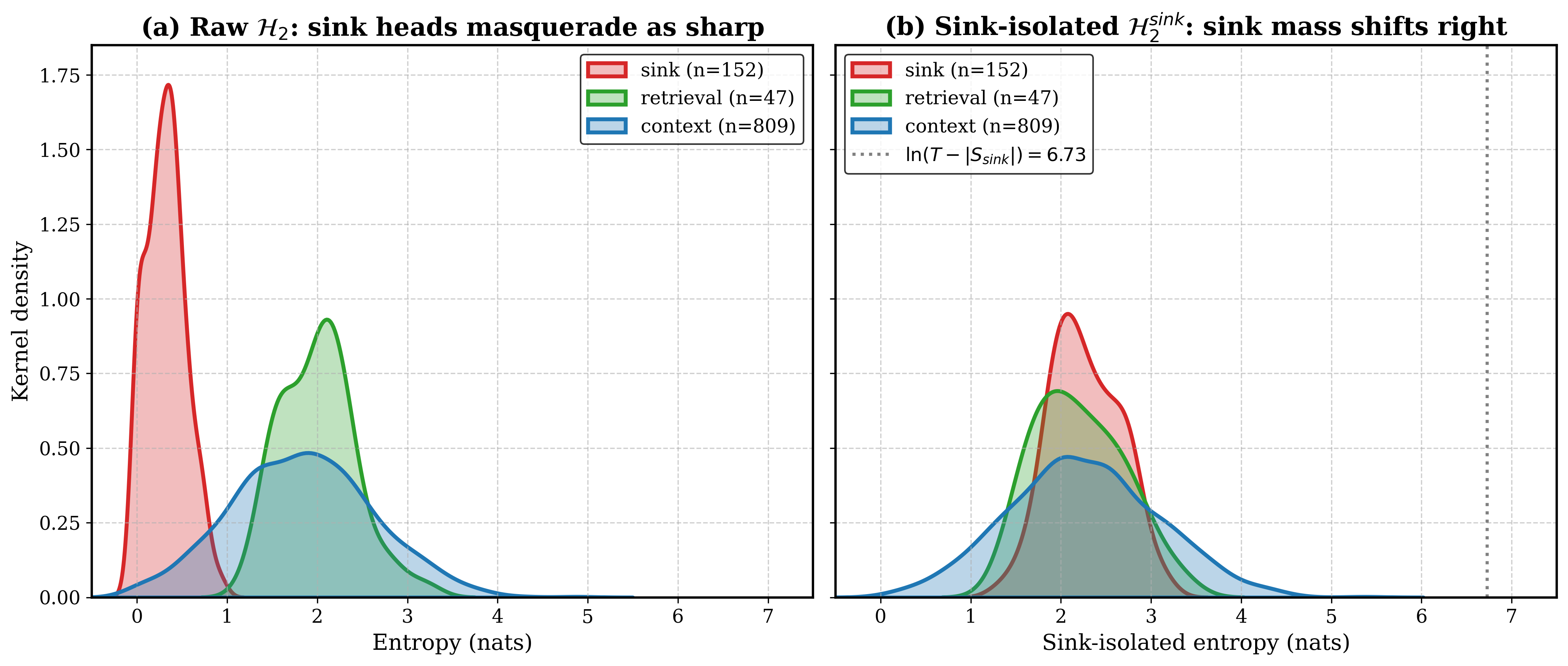}
\caption{Sink isolation separates head populations
  (\cref{subsec:pilot-ablations}; Qwen2.5-1.5B pilot; three probes,
  1{,}008 layer--head observations: $n{=}152$ sink, $n{=}47$ retrieval,
  $n{=}809$ contextual). (a)~Raw collision entropy: the sink population
  sits at low entropy and overlaps the left flank of the retrieval
  population---the masquerade that \cref{thm:sink_suppression} predicts
  raw entropy cannot break. (b)~Sink-isolated entropy: the sink mass
  shifts right while retrieval heads stay low. The dashed line marks
  the uniform-residual ceiling $\ln(T-|\Ssink|) = 6.73$ of
  \cref{thm:sink_suppression}; measured populations sit below it
  because real sink residuals are not uniform. Isolation inverts the
  sink-to-context weight ordering rather than driving any single weight
  to zero (\cref{fig:pilot_weights}).}
\label{fig:pilot_sink}
\end{figure*}

\begin{figure*}[t]
\centering
\includegraphics[width=0.7\textwidth]{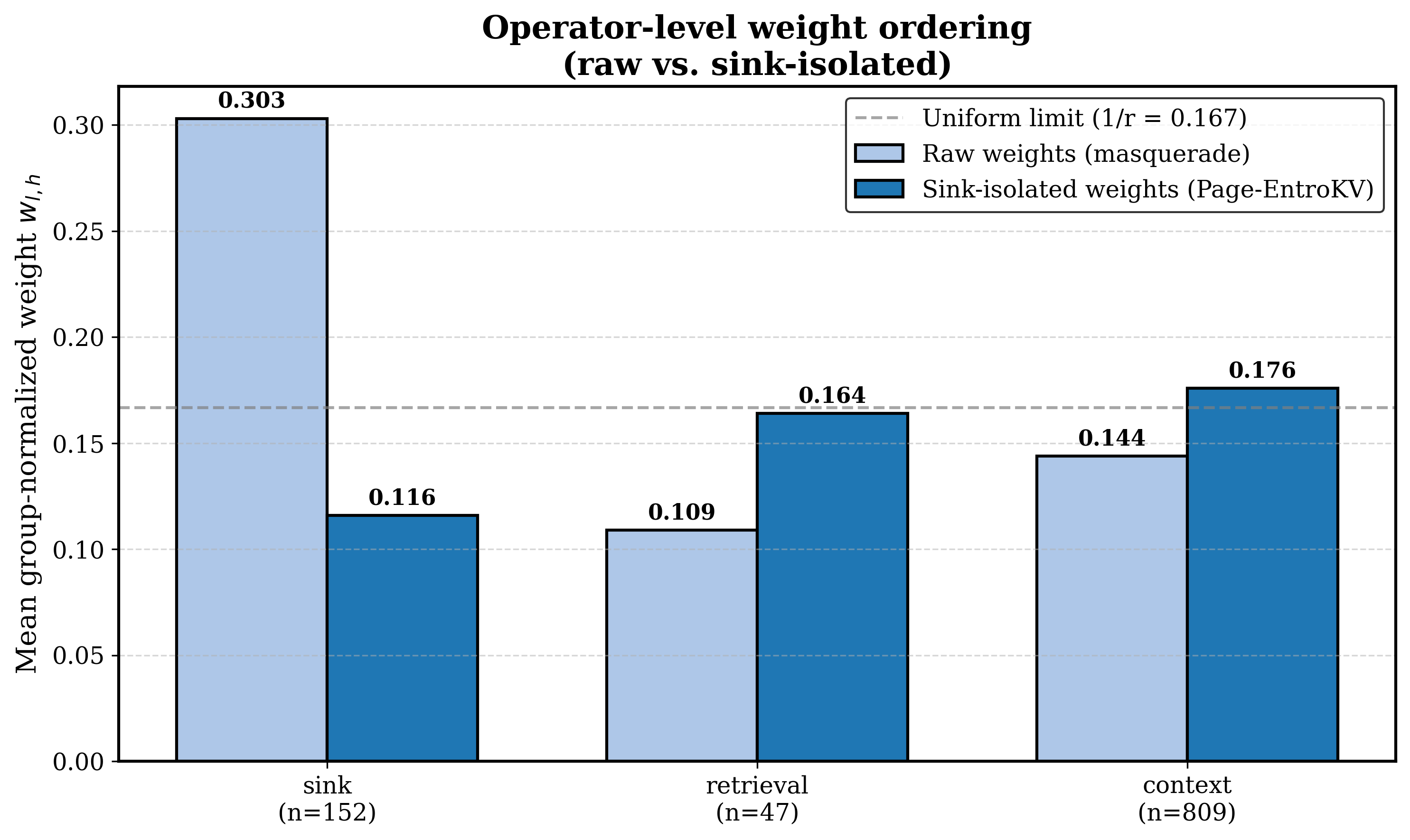}
\caption{Operator-level pooling weights under raw versus sink-isolated
  entropy, from the same 1{,}008-observation census as
  \cref{fig:pilot_sink} (\cref{subsec:pilot-ablations}). Sink isolation
  lowers the sink head's weight from $0.303$ to $0.116$ and raises the
  contextual head's weight from $0.144$ to $0.176$, inverting the
  sink-to-context weight ratio from $2.10\times$ (raw: sink heads
  outweigh contextual heads without earning it) to $0.66\times$   (isolated: contextual heads dominate)---the empirical form of
  \cref{thm:sink_suppression}.}
\label{fig:pilot_weights}
\end{figure*}

\subsection{Retrieval and downstream feasibility}
\label{subsec:pilot-retrieval}

\Cref{tab:pilot_needle} reports needle retention under two success
criteria, because they answer different questions and were measured in
separate runs. Under the \emph{pooled-selection criterion}---a trial
counts as retained if the needle's page survives the pooled-selection
eviction decision, with no requirement across layers---entropy-weighted
pooling retains the needle at $100\%$ across all pilot contexts and
depths while arithmetic-mean pooling retains it at $0\%$: the dilution
mechanism of \cref{eq:mean_dilution} is directly observable. Under the
\emph{strict all-layer criterion}---the needle's page must survive in
\emph{every} layer---both methods fail at depths $25$--$75\%$ (weighted
recall $0$--$33\%$ across strict variants; mean pooling $0\%$ throughout), and both retain trivially at depth $100\%$, where the
needle sits in the final pages that the quota never evicts. Every zero
cell was verified as a genuine eviction: in 30 of 30 zero cells the
needle page's rank fell below the retention quota in the layer that
dropped it. The uncompressed cache retains the needle at all depths
under both criteria, so the failures are budget effects, not harness
artifacts.

The theorem's premise was measured, not assumed. On the strict-criterion
pilot ($T\in\{1024,2048\}$), needle attention mass is
 $1-\epsilon \in [0.11,\,0.17]$---roughly two orders of magnitude above
the uniform floor $1/T \approx 4.9\times 10^{-4}$---with measured
sink-isolated retrieval entropy $\eta_{\mathrm{ret}}\approx 2.6$--$2.8$ nats and pooled needle mass $\approx 1.0$--$1.6\times 10^{-3}$. The
retrieval-head premise therefore holds directionally at pilot scale,
but the strong regime the theorem asymptotes to ($1-\epsilon \to 1$) is
not reached: pooled mass of $\sim 10^{-3}$ cannot top a $20\%$ quota in
every one of 28 layers. The strict-criterion failures at depths
 $25$--$75\%$ are the pilot-scale face of this gap, and closing it
(per-layer budget scaling, or criteria matched to the mass
distribution) is the first item of the pending sweep. Depth-$0$ needles fall inside the sink region, which \cref{def:sinkentropy}
isolates and \cref{rem:assumptions} explicitly excludes from the
theorem. The engine path reproduces the \cref{cor:finite_context_bound}
bound to floating-point tolerance under a synthetic premise-satisfied
construction (\cref{fig:cor2}a); on real text at pilot scale the
premise is unsatisfied and pooling degenerates toward the mean
(\cref{fig:cor2}b), consistent with \cref{tab:pilot_needle}. On
premise-holding layers the measured bound deficit $\Delta(T)$ shrinks
from $1.3\times 10^{-2}$ at $T{=}1024$ to $8\times 10^{-4}$ at
 $T{=}8192$---the monotone decay the corollary predicts---while the
absolute premise gap keeps the real-text pooled mass orders below the
bound at pilot scale.

\begin{table}[t]
\centering\small
\caption{Needle retrieval under two success criteria
 (Qwen2.5-1.5B pilot, matched $20\%$ physical budget). Pooled-selection
 criterion: retention requires the needle's page to survive the pooled
 eviction decision (no cross-layer requirement). Strict all-layer
 criterion: the needle's page must survive in every layer. Asterisk:
 weighted recall at depth~0 spans $0$--$33\%$ across strict-variant
 runs. All zero cells verified as genuine evictions (30/30).}
\label{tab:pilot_needle}
\begin{tabular}{@{}lcccc@{}}
\toprule
\multicolumn{5}{@{}l}{\emph{(a) Pooled-selection criterion}}\\
\midrule
Context $T$ & 959 & 1469 & 1989 & 2499 \\
\midrule
Arithmetic mean & 0.0\% & 0.0\% & 0.0\% & 0.0\% \\
\textbf{\method{}} & \textbf{100.0\%} & \textbf{100.0\%} &
\textbf{100.0\%} & \textbf{100.0\%} \\
\midrule
\multicolumn{5}{@{}l}{\emph{(b) Strict all-layer criterion (recall by
  needle depth)}}\\
\midrule
Method & 0\% & 25\% & 50\% & 75\% / 100\% \\
\midrule
Uncompressed & 1.00 & 1.00 & 1.00 & 1.00 / 1.00 \\
Mean pooling (page) & 0.00 & 0.00 & 0.00 & 0.00 / 1.00 \\
\textbf{\method{}} & 0--0.33$^{*}$ & 0.00 & 0.00 & 0.00 / 1.00 \\
\bottomrule
\end{tabular}
\end{table}

\Cref{tab:pilot_downstream} reports the four LongBench subsets that
completed under the pilot harness at $20\%$ retention (20 samples per
subset): representative QA and code tasks remain solvable
(qasper 38.60\%, multifieldqa\_en 44.20\%, lcc 56.10\%,
repobench-p 44.81\%). Three caveats bound the claim: the subsets are
20-sample slices, there is no uncompressed same-harness baseline yet,
and the eleven subsets that faulted under an answer-extraction bug are
\emph{excluded}, not reported as zeros.

\begin{table}[t]
\centering\small
\caption{Downstream feasibility (LongBench subsets that completed,
 $20\%$ retention, 20 samples each; Qwen2.5-1.5B pilot).
 Scores are not comparable to published LongBench numbers (sample size,
 retention, and the absence of an uncompressed same-harness baseline).
 The remaining LongBench subsets are pending re-run after an evaluator
 fault; unscored subsets are excluded, not reported as zero.}
\label{tab:pilot_downstream}
\begin{tabular}{@{}lc@{}}
\toprule
LongBench subset & \method{} score (\%) \\
\midrule
qasper          & 38.60 \\
multifieldqa\_en & 44.20 \\
lcc             & 56.10 \\
repobench-p     & 44.81 \\
\bottomrule
\end{tabular}
\end{table}

\subsection{Remaining pre-registered protocol}
\label{subsec:pilot-pending}

The following protocol items are pre-registered and in progress; no
claim about them is made anywhere in this paper:
\begin{itemize}
  \item \textbf{Phase~1 scale-up:} Llama-3.1-8B-Instruct ($r{=}4$;
   $L{=}32$; $\HKV{=}8$; $\Dhead{=}128$)~\cite{meta2024llama3} on a
   single NVIDIA A100-80\,GB GPU under
   vLLM~\cite{kwon2023vllm}, with the serving tables (maximum stable
   batch, decoding throughput, time-to-first-token, inter-token
   latency) and the pre-registered closure check; extensions to
   Llama-3.1-70B, Qwen-2.5-7B, and Qwen-2.5-72B.
  \item \textbf{Long-context NIAH:} the full sweep over
   $\{4\text{k},\dots,128\text{k}\}$ at depths
   $\{0\%,25\%,50\%,75\%,100\%\}$---needle-recall horizons in the pilot
   are $\le 2.5\text{k}$ only---with reconciled success criteria across
   the pilot variants.
  \item \textbf{LongBench re-run:} all fifteen subsets with a repaired
   evaluator and an uncompressed same-harness baseline.
  \item \textbf{Strict all-layer retention at scale:} framed as the
   regime where the premise of \cref{thm:needle_preservation}
   ($1-\epsilon$ large) binds; to be tested in Phase~1, where retrieval
   heads are sharper and budgets are set per the instantiated error
   bound of \cref{cor:finite_context_bound}.
  \item \textbf{Remaining ablations:} block-size effects on recall
   and sink isolation on sink-sensitive task subsets.
\end{itemize}

\subsection{Logging}
\label{subsec:pilot-logging}

All runs append to a CSV log with fields
\begin{quote}\small
\texttt{experiment, model, method, budget, context, subset, metric,
value, unit, hardware, seed, timestamp},
\end{quote}
and every table in this paper is transcribed from that log: the
structural replay (\texttt{experiment\_1\_log.csv}, 2{,}240 rows; the
by-scorer replay behind \cref{tab:pilot_uor}a), the needle pilot
(\texttt{experiment\_3\_log.csv}, 321 rows), the ablation logs (584 and
201 rows, with the certified run behind \cref{tab:pilot_ablations}),
the three-probe sink census (1{,}008 observations, behind
\cref{fig:pilot_sink,fig:pilot_weights}), and the downstream summary
(\texttt{experiment\_5\_longbench\_summary.csv}). Broken runs are
discarded from the workspace manifest and tagged as non-transcription
sources. Headline quantities are defined at the table level and
propagate through macros, so the abstract, tables, and conclusion
cannot disagree.

\section{Discussion and limitations}
\label{sec:discussion}

\paragraph{Failure modes.}
When all heads in a group exhibit near-identical high entropy (uniformly
diffuse attention), the Boltzmann weights degenerate toward uniform
(\cref{prop:temperature_limits}~ii), and \method{} reduces to mean
pooling---re-opening the needle-dilution risk
(\cref{rem:assumptions}). The fallback sanctioned by
\cref{prop:temperature_limits}~(iii) is to switch a group to
max-over-heads pooling when intra-group entropy variance falls below a
threshold; the switch is cost-neutral for the budget guarantee, which
\cref{thm:memory_guarantee} ties only to pooling before selection.

\paragraph{Pilot limitations.}
Three limitations scope the empirical tier. First, \emph{strict
all-layer needle retention fails at pilot scale} at depths
 $25$--$75\%$: the premise of \cref{thm:needle_preservation} holds
directionally but not in its strong regime, so pooled needle mass
($\sim 10^{-3}$) cannot top a $20\%$ quota in every layer; the
asymptotic guarantee and deployed budgets are reconciled only at
Phase-1 scale (\cref{subsec:pilot-pending}), and the two-regime check
of \cref{fig:cor2} delimits exactly where the guarantee binds. Second,
\emph{needle-recall horizons are $\le 2.5$k tokens}; the 4k--128k sweep
is pending. Third, \emph{the pilot covers a single model family}
(Qwen2.5-1.5B, $r{=}6$); cross-family generality is deferred to
Phase~1, and no serving-level measurements exist yet.

\paragraph{Kernel overhead.}
The reference implementation computes entropy, weights, pooling, and
page reduction as separate operations. The two-kernel fused design of
\cref{subsec:systems}---warp-level dot products for sink-isolated
collision entropy, register-resident weighting, page-by-page block-max,
and in-kernel radix top-$K$ selection
(\cref{prop:complexity,rem:smem,rem:topk})---is what reaches the
throughput characteristics the framework is designed for; its
open-source release is future work.

\paragraph{Runtime configuration dependency.}
The framework's guarantees are stated relative to the engine's block
size and page table: $K_l$ scales with $1/B$ (\cref{eq:page_quotas}),
the page-level union ratio is bounded by $\HKV$ (\cref{prop:engine_page_accounting}), and selection cost scales with
 $T/B$ (\cref{prop:complexity}). Deployments must pin the engine version
and block size together with the eviction configuration; the protocol
records both with every log.

\paragraph{Orthogonality to KV-cache quantization.}
\method{} operates exclusively along the sequence dimension $T$,
selecting which pages to retain; low-bit KV quantization
(KIVI~\cite{liu2024kivi}, QuaRot~\cite{ashkboos2024quarot}) operates
along the channel dimension, reducing the bytes per stored element
from $2$ (FP16) to $0.5$ (INT4). The two optimizations compose
multiplicatively: a $20\%$ page budget yields a $5\times$ reduction
along $T$ by construction, and INT4 yields a further $4\times$ along
the element axis---$20\times$ combined relative to the uncompressed
FP16 cache. The pooling and selection logic is unaffected by storage
precision: importance scores are computed from attention probabilities
during the QK dot product, which are accumulated in floating point
regardless of the stored representation; quantization changes only the
stored $\mathbf{K}$ and $\mathbf{V}$ tensors. This is a design-level
composition argument, not a measured deployment result.

\paragraph{Scope boundary.}
\method{} targets GQA architectures. It does not transfer directly to
latent-compression designs such as multi-head latent
attention~\cite{deepseek2024deepseekv2}, where KV states are collapsed
into a shared latent vector and no per-token page structure exists to
evict. For MHA ($r=1$) the union bound is vacuous; entropy-weighted
pooling and page alignment remain applicable, but the memory benefit
comes from page alignment rather than union elimination.

\paragraph{Empirical status.}
The guarantees of \cref{thm:memory_guarantee,prop:engine_page_accounting}
are architectural and hold by construction; the pilot confirms their
measured form (exact cardinality in 192 of 192 operator--block cells;
page-union ratio at the bound's lower endpoint). The pilot further
supports the structural analysis on real text: heavy-tailed inflation
up to $5.1\times$ with envelope-certified twofold overhead
(\cref{tab:pilot_uor}), sink-isolation ordering inversion
(\cref{fig:pilot_sink,fig:pilot_weights}), and the temperature and
R\'enyi-order characterizations, certified over the full layer--group
population (\cref{tab:pilot_ablations}). Measured but pending full
validation: pooled-criterion needle retention
(\cref{tab:pilot_needle}a), downstream feasibility at $20\%$ retention
(\cref{tab:pilot_downstream}), and the strict all-layer regime, which
fails at pilot scale for the premise reasons recorded in
\cref{subsec:pilot-retrieval}. Everything at Phase-1 scale---128k
contexts, larger models, uncompressed baselines, and serving
throughput---is pre-registered and pending
(\cref{subsec:pilot-pending}). Until those runs are logged, the
appropriate reading is a formal framework whose structural core is
pilot-validated, not an empirically validated serving system.

\section{Conclusion}
\label{sec:conclusion}

We presented \method{}, a formal framework for hardware-aligned
KV-cache eviction that resolves the mismatch between head-independent
eviction policies and physically shared GQA buffers. The framework
formalizes the union overhead ratio and intra-group disagreement,
proves their exact analytic relationship at $r=2$, and bounds the
inflation two-sidedly at every group ratio; it replaces head-level
selection with sink-isolated entropy-weighted pooling at the physical
group level; and it projects importance onto the discrete page frames
that serving engines allocate, so eviction executes exactly at the
tuple $(l, g, P_b)$ with strict, provable budget preservation and
finite-context needle retention. A pilot validation on
Qwen2.5-1.5B-Instruct confirms the structural core on real
text---heavy-tailed union inflation with envelope-certified overhead,
sink-isolation ordering inversion, exact pooled cardinality, and the
certified temperature and R\'enyi-order characterizations---while the
full-scale sweep proceeds under the pre-registered protocol. Future
work: executing that protocol across model families, releasing the
fused kernels, entropy-based inter-layer budget allocation beyond the
present uniform quotas, and training-time co-design of page-granular
attention.

\section*{CRediT authorship contribution statement}
\textbf{Inbasekaran S.:} Conceptualization, Methodology, Software,
Validation, Formal analysis, Investigation, Data curation, Writing --
original draft, Writing -- review \& editing, Visualization.

\section*{Declaration of competing interest}
The author declares no known competing financial interests or personal
relationships that could have appeared to influence the work reported in
this paper.

\section*{Declaration of generative AI and AI-assisted technologies
in the writing process}
During the preparation of this work the author used a generative-AI
assistant for language editing and manuscript preparation. After using
the tool, the author reviewed and edited the content as needed and
takes full responsibility for the content of the published article.

\section*{Data availability}
All experiment logs---the per-cell CSVs for the structural replay
(2{,}240 rows), the by-scorer replay, the needle pilot (321 rows), the
ablation studies (584 and 201 rows, including the certified run), the
three-probe sink census (1{,}008 observations), and the downstream
summary---the evaluation harness, and the figure-generation scripts are
released at \url{https://github.com/bruce12-glitch/PageEntro-KV}.
Summary tables in this paper are transcribed directly from those logs;
broken runs are excluded from the manifest and are not transcription
sources.

\appendix

\section{Notation}
\label{app:notation}

\Cref{tab:notation} summarizes the symbols, architectural parameters,
and systems notation used throughout the paper.

\begin{table*}[t]
\centering
\small
\caption{Symbols and notation used throughout the paper.}
\label{tab:notation}
\begin{tabularx}{\textwidth}{@{} l X @{}}
\toprule
\textbf{Symbol} & \textbf{Definition} \\
\midrule
 $L$ & Number of transformer layers, $l \in \{1,\dots,L\}$. \\
 $\HQ,\ \HKV$ & Logical query heads; physical KV heads. \\
 $r=\HQ/\HKV$ & GQA group ratio (query heads per physical KV buffer). \\
 $\Dmodel,\ \Dhead$ & Hidden dimension; head dimension
  ($\Dmodel = \HQ \cdot \Dhead$). \\
 $T$ & Current sequence length ($T \le T_{\mathrm{ctx}}$). \\
 $B$ & PagedAttention block size (tokens per page, e.g.\ $B=16$). \\
 $\Nblk$ & Number of allocated physical blocks. \\
 $\PageSet,\ P_b$ & Set of physical pages; page frame $b$. \\
 $\Attn{l}{h}\in\mathbb{R}^{T}$ & Attention probability vector, head $h$,
  layer $l$. \\
 $\TokSet{h}$ & Token-level top-$k$ selection set of head $h$   (\cref{def:uor}); page-level retention sets $\RetSet{l}{g}$ are
  distinguished by their group subscript. \\
 $I_{ij},\ J_{ij},\ \mathcal{D}_{\mathrm{Jaccard}},\ \bar{J}$ &
  Pairwise intersection size; Jaccard similarity and disagreement; mean
  pairwise similarity (\cref{def:jaccard,cor:bonferroni}). \\
 $\Ssink,\ \tilde{\mathbf{A}}_{l,h}$ & Sink token set ($|\Ssink|=4$);
  sink-isolated attention distribution (\cref{def:sinkentropy}). \\
 $\Renyi{\alpha}{A}$ & Discrete R\'enyi entropy of order $\alpha$. \\
 $\Collision{A}$ & Collision entropy (R\'enyi $\alpha{=}2$):
  $-\ln\|A\|_2^2$. \\
 $\CollisionSink{A}$ & Sink-isolated collision entropy. \\
 $W_{\mathrm{obs}},\ W_{\max}$ & Prefill observation window and its
  bound ($W_{\max}=4096$); entropies for head weighting are evaluated
  over $W_{\mathrm{obs}}$ (\cref{def:obswindow}). \\
 $\mathcal{B}_{\mathrm{total}},\ \mathcal{B}_l$ & Total KV capacity;
  uniform per-layer budget (\cref{eq:page_quotas}). \\
 $K_l$ & Pages retained per physical group in layer $l$   (\cref{eq:page_quotas}). \\
 $\Group{g}$ & Set of query-head indices mapping to KV group $g$. \\
 $w_{l,h}$ & Entropy-weighted importance of head $h$ within its group. \\
 $\AttnPool{l}{g}\in\mathbb{R}^T$ & Group-level pooled attention. \\
 $\PageImp{l}{g}{b}$ & Importance of page $P_b$ in group $g$, layer $l$. \\
 $\RetSet{l}{g}$ & Page-level retention set of group $g$, layer $l$. \\
 $\UOR(g,k)$ & Union overhead ratio under head-independent top-$k$. \\
 $\tau_g$ & Group-pooling temperature (\cref{eq:boltzmann_weights}). \\
\bottomrule
\end{tabularx}
\end{table*}

\section{Master verification summary}
\label{app:verification}

\Cref{tab:master_verification_summary} maps every theoretical guarantee
and hardware-geometry specification in the paper to its mathematical
form, derived value or bound, systems-level impact, and the pilot
evidence (or design-spec status) that backs it.

\begin{table*}[t]
\centering
\small
\caption{Master verification summary of \method{} theoretical
  guarantees and hardware geometry (Tables~\ref{tab:pilot_uor}--\ref{tab:pilot_downstream};
  Figs.~\ref{fig:pilot_uor}--\ref{fig:pilot_weights}).}
\label{tab:master_verification_summary}
\begin{tabular}{@{}p{0.13\textwidth} p{0.22\textwidth} p{0.19\textwidth} p{0.18\textwidth} p{0.15\textwidth}@{}}
\toprule
\textbf{Guarantee} & \textbf{Mathematical form} &
\textbf{Value / bound} & \textbf{Systems impact} &
\textbf{Validated by} \\
\midrule
Arithmetic intensity & $W_{\mathrm{FLOPs}}/D_{\mathrm{bytes}}$  (\cref{eq:arithmetic_intensity}) & $2r/b_{\mathrm{el}} = 4.0$  FLOP/Byte & Decode phase memory-bound & Architectural constant \\
H100 ridge point & \cref{eq:ridge_point} & $\approx 295.4$ FLOP/Byte &
Decode at $<1.5\%$ of peak & Architectural constant \\
Cache footprint & \cref{eq:kv_footprint} & $4\Bbatch L\HKV\Dhead T$  Bytes & Linear HBM capacity & Architectural constant \\
 $r=2$ identity & $\UOR = 2/(2-D)$, monotone and convex in $D$  (\cref{rem:phi}) & Exact & Disagreement implies inflation &
 Table~\ref{tab:pilot_uor} certificates \\
Bonferroni certificate & $\UOR \ge \max\{1,\, r - r(r-1)\bar{J}\}$ &
 $2\times$ at $\bar{J}\le\frac{r-2}{r(r-1)}$ & Elementary certificate &
 Table~\ref{tab:pilot_uor} ($\mathrm{cert}_1$) \\
Sandwich bound & $\frac{r}{1+2(r-1)\bar J/(1+\bar J)} \le \UOR \le
 \min(r,\, r - \frac{2J_{\max}}{1+J_{\max}})$ & Tight at extremes &
 Two-sided inflation bound & Table~\ref{tab:pilot_uor}
 ($\mathrm{cert}_T$); \cref{fig:pilot_uor} \\
R\'enyi family & Shannon / collision / min-entropy limits &
 $\mathcal{H}_2=-\ln\|A\|_2^2$ & Dot-product evaluation &
 Table~\ref{tab:pilot_ablations}b (zero violations) \\
Sink suppression & $w_{\mathrm{sink}} \propto N^{-1/\tau_g}$ &
 Sink ordering inverted $2.10\times \to 0.66\times$ & No sink
 masquerade & \cref{fig:pilot_sink,fig:pilot_weights} \\
Mean dilution & $\AttnMean{l}{g}(t^{*}) \to (1-\epsilon)/r$ & Exact
 $1-1/r$ drop & Mean pooling failure mode &
 Table~\ref{tab:pilot_needle}a (mean $0\%$) \\
Temperature limits & $\tau_g \to \infty \Rightarrow 1/r$; hot limit
 one-hot & Monotone in $\tau_g$ & Interpolates aggregators &
 Table~\ref{tab:pilot_ablations}a (certified) \\
Budget conservation & \cref{eq:page_quotas} + Hamilton rounding &
 Exact iff $\HKV B \mid \mathcal{B}_{\mathrm{total}}$ & Static page
 quotas & Table~\ref{tab:pilot_ablations}c (exact cardinality) \\
Engine page accounting & $K_l \le |P^{\mathrm{ret}}_l| \le \HKV K_l$ &
 Per-layer cross-group & Residual overhead bounded &
 Table~\ref{tab:pilot_ablations}d (PUR $1.0$) \\
Memory guarantee & $M_{\mathrm{pooled}} \equiv k$ & $\UOR \equiv 1.0$ &
 Union overhead eliminated & Table~\ref{tab:pilot_uor} pooled row \\
Needle preservation & pooled mass $\to 1-\epsilon$ & Premise-gated &
 Retrieval signal preserved & Table~\ref{tab:pilot_needle};
 \cref{fig:cor2} \\
Finite bound & $\Delta(T)$ & $\le 1.64\times 10^{-9}$ at 128k
 (instantiation) & Non-asymptotic error budget &
 \cref{tab:finite_scaling,fig:cor2} \\
Page monotonicity & block-max preserves argmax page & Exact & Token to
 page linkage & Design property \\
Fused complexity & $\mathcal{O}(\HQ T + \HKV\tfrac{T}{B}\log K_l)$ &
 On-chip & No DRAM intermediates & Design specification \\
Shared-memory footprint & $2rB + 2(r{+}B)\Dhead + 4\HKV$ & $6{,}432$  Bytes ($r{=}8$, $B{=}16$) & Fits H100 SM ($228$\,KB) & Design
 specification \\
Radix top-$K$ & two-pass 16-bit histogram &
 $\mathcal{O}(N_{\mathrm{pages}})$ & Exact $K_l$ cardinality & Design
 specification \\
FP16 stability & guard inside $\ln$ only; clamped denominator &
 $\|\tilde A\|_2^2 \in [1/T,1]$ & No $\ln(0)$; degenerate sinks
 suppressed & Table~\ref{tab:pilot_ablations}b (zero violations) \\
\bottomrule
\end{tabular}
\end{table*}

\section{Step-by-step derivations}
\label{app:derivations}

This appendix collects the full step-by-step derivations supporting the
main-text results. Only elementary set algebra, Cauchy--Schwarz,
Jensen's inequality, and L'H\^opital's rule are used.

\subsection{The GQA inflation sandwich bound}
\label{app:sandwich}

\paragraph{Step 1 (coverage degrees).}
Let $U = \bigcup_{i=1}^r \TokSet{h_i}$. For $m \in U$ define the
multiplicity
 $x_m \triangleq \sum_{i=1}^r \mathbf{1}\{m \in \TokSet{h_i}\}$: every
retained token is selected by at least one head and at most $r$, so
 $x_m \in \{1,\dots,r\}$.

\paragraph{Step 2 (exact moments over $U$).}
Exchanging the order of finite summations,
\[
  \sum_{m \in U} x_m
  = \sum_{i=1}^r \sum_{m \in U} \mathbf{1}\{m \in \TokSet{h_i}\}
  = \sum_{i=1}^r |\TokSet{h_i}| = rk,
\]
and, using the indicator product rule
 $\mathbf{1}_A\mathbf{1}_B = \mathbf{1}_{A\cap B}$ and splitting the
double sum into diagonal ($i=j$) and symmetric off-diagonal ($i\neq j$)
pairs,
\[
  \sum_{m \in U} x_m^2
  = \sum_{i=1}^r\sum_{j=1}^r |\TokSet{h_i}\cap\TokSet{h_j}|
  = rk + 2\sum_{i<j} I_{ij},
  \qquad I_{ij} \triangleq |\TokSet{h_i}\cap\TokSet{h_j}|.
\]

\paragraph{Step 3 (Cauchy--Schwarz).}
Applying Cauchy--Schwarz to $(x_m)_{m\in U}$ and the all-ones vector,
\[
  \Bigl(\sum_{m\in U} x_m\Bigr)^{2}
  \le \Bigl(\sum_{m\in U} 1\Bigr)\Bigl(\sum_{m\in U} x_m^{2}\Bigr),
\]
and $\sum_{m\in U} 1 = |U|$, so
 $(rk)^{2} \le |U|\,(rk + 2\sum_{i<j} I_{ij})$. Isolating $|U|$ and
dividing by $k$:
\[
  \UOR(g,k) = \frac{|U|}{k}
  \;\ge\; \frac{r^{2}}{r + \tfrac{2}{k}\sum_{i<j} I_{ij}}.
\]

\paragraph{Step 4 (pairwise overlaps via Jaccard).}
Inclusion--exclusion gives
 $|\TokSet{h_i}\cup\TokSet{h_j}| = 2k - I_{ij}$, hence
 $J_{ij} = I_{ij}/(2k - I_{ij})$. Solving,
 $2kJ_{ij} = I_{ij}(1+J_{ij})$, so
\[
  I_{ij} = \frac{2kJ_{ij}}{1+J_{ij}}
  \qquad\Longrightarrow\qquad
  \frac{2}{k}\sum_{i<j} I_{ij}
  = 4\sum_{i<j} \frac{J_{ij}}{1+J_{ij}}.
\]

\paragraph{Step 5 (strict concavity and Jensen).}
 $f(x) \triangleq \frac{x}{1+x} = 1 - (1+x)^{-1}$ satisfies
 $f'(x) = (1+x)^{-2}$ and $f''(x) = -2(1+x)^{-3} < 0$ on $[0,1]$:
strictly concave. Under the uniform weights $2/(r(r-1))$ across the
 $r(r-1)/2$ pairs,
\[
  \frac{2}{r(r-1)}\sum_{i<j} f(J_{ij})
  \;\le\; f(\bar J) = \frac{\bar J}{1+\bar J}.
\]
Multiplying by $r(r-1)$ bounds the denominator cross-term:
 $4\sum_{i<j}\frac{J_{ij}}{1+J_{ij}} \le 2r(r-1)\frac{\bar J}{1+\bar J}$.

\paragraph{Step 6 (lower bound completed).}
Substituting Steps 4--5 into Step 3 (upper-bounding the denominator
lower-bounds the fraction) and factoring $r$ from numerator and
denominator:
\[
  \UOR(g,k) \;\ge\; \frac{r^{2}}{r + 2r(r-1)\frac{\bar J}{1+\bar J}}
  \;=\; \frac{r}{1 + 2(r-1)\frac{\bar J}{1+\bar J}}.
\]

\paragraph{Step 7 (upper bound).}
Subadditivity gives $|U| \le rk$, i.e., $\UOR \le r$. To refine using
pairwise data, isolate an arbitrary pair $\{h_i,h_j\}$:
 $U = (\TokSet{h_i}\cup\TokSet{h_j}) \cup
\bigl(\bigcup_{m\neq i,j}\TokSet{h_m}\bigr)$, so
 $|U| \le (2k - I_{ij}) + (r-2)k = rk - I_{ij}$ and, by Step 4,
 $\UOR \le r - \frac{2J_{ij}}{1+J_{ij}}$ for every pair. Taking the
infimum over all pairs subtracts the maximal term, and intersecting
with $\UOR \le r$ yields the upper bound of \cref{eq:sandwich_bound}.

\paragraph{Step 8 (extremes and the certification threshold).}
 $\bar J = 0$: disjoint selections, both bounds equal $r$---exact.
 $\bar J = 1$: identical selections, the lower bound evaluates to
 $r/(1 + (r-1)) = 1$---exact; the upper bound evaluates to $r-1$ for
 $r>2$, valid but loose because the pairwise profile cannot see
higher-order overlap. Solving
 $\frac{r}{1+2(r-1)\frac{\bar J}{1+\bar J}} \ge 2$ gives the
certification threshold $\bar J \le \frac{r-2}{3r-2}$.

\subsection{The elementary Bonferroni bound}
\label{app:bonferroni}

The full inclusion--exclusion expansion of $|U|$ truncated after the
second-order sum is a strict lower bound (second-order Bonferroni):
 $|U| \ge rk - \sum_{i<j}I_{ij}$. Since
 $|\TokSet{h_i}\cap\TokSet{h_j}| = J_{ij}
|\TokSet{h_i}\cup\TokSet{h_j}| \le 2kJ_{ij}$ by subadditivity of the
pairwise union, summing over all $r(r-1)/2$ pairs and using
 $\sum_{i<j}J_{ij} = \frac{r(r-1)}{2}\bar J$ gives
 $\sum_{i<j}I_{ij} \le kr(r-1)\bar J$. Dividing by $k$ and clamping with
 $\UOR \ge 1$ (the union contains any single head's $k$ selections)
yields \cref{eq:bonferroni_bound}. Non-trivial inflation requires
 $r\bar J < 1$, i.e., $\bar J < 1/r$; twofold certification requires
 $\bar J \le \frac{r-2}{r(r-1)}$, i.e., $\bar J \le \frac16$ for $r=4$ and $\bar J \le \frac{3}{28}$ for $r=8$.

\subsection{Limit orders of the R\'enyi family}
\label{app:renyi}

Write $A \equiv \Attn{l}{h}$, $\psi(\alpha) \triangleq \sum_{t=1}^T
A(t)^{\alpha}$, $f(\alpha) \triangleq \ln\psi(\alpha)$, and
 $g(\alpha) \triangleq 1-\alpha$.

\paragraph{Part 1: Shannon ($\alpha \to 1$).}
 $\psi(1) = 1$, so $f(1) = 0$ and $g(1) = 0$: the ratio is a
 $\frac00$ form. Both functions are continuously differentiable and
 $g'(\alpha) \equiv -1 \neq 0$, so L'H\^opital's rule applies once
 $\lim f'/g'$ exists. Termwise differentiation via
 $\frac{d}{d\alpha}x^{\alpha} = x^{\alpha}\ln x$ (with the continuous
extension $p^{\alpha}\ln p \to 0$ as $p \to 0^{+}$) gives
 $f'(\alpha) = \frac{\sum_{t} A(t)^{\alpha}\ln A(t)}
{\sum_{t} A(t)^{\alpha}}$. As $\alpha \to 1$, the denominator tends to
 $1$ and each $A(t)^{\alpha}\ln A(t) \to A(t)\ln A(t)$, so
 $\lim_{\alpha\to1}\Renyi{\alpha}{A} = -\sum_{t=1}^{T} A(t)\ln A(t)$.

\paragraph{Part 2: collision ($\alpha = 2$).}
Direct substitution:
 $\Renyi{2}{A} = -\ln\sum_t A(t)^2 = -\ln\|A\|_2^2$. The self-inner
product $\langle A, A\rangle = \|A\|_2^2 \in (0,1]$ is the collision
probability.

\paragraph{Part 3: min-entropy ($\alpha \to \infty$).}
Let $M = \max_t A(t) \in [\frac1T, 1]$ and
 $\mathcal{T}_{\max} = \{t : A(t) = M\}$ with multiplicity
 $m \ge 1$; on the complement, $0 \le A(t)/M < 1$ strictly. Factoring
 $M^{\alpha}$ out of the power sum and using $\ln(uv) = \ln u + \ln v$:
\[
  \Renyi{\alpha}{A}
  = \frac{\alpha}{1-\alpha}\ln M
  + \frac{1}{1-\alpha}\ln\Bigl[m
      + \sum_{t\notin\mathcal{T}_{\max}}
        \Bigl(\frac{A(t)}{M}\Bigr)^{\alpha}\Bigr].
\]
As $\alpha \to \infty$: each $(A(t)/M)^{\alpha} \to 0$, so the inner
logarithm tends to $\ln m$; $\frac{\alpha}{1-\alpha} \to -1$; and
 $\frac{1}{1-\alpha}\ln m \to 0$. Hence
 $\lim_{\alpha\to\infty}\Renyi{\alpha}{A} = -\ln M$.

\subsection{Sink head suppression}
\label{app:sink}

\paragraph{Step 1 (conditional distribution).}
Partition $\{1,\dots,T\}$ into $\Ssink$ and
 $S_{\mathrm{ctx}}$ with $|S_{\mathrm{ctx}}| = N$. By normalization,
 $\sum_{t\in S_{\mathrm{ctx}}}\Attn{l,h}(t) = 1 - p_h$, so the
renormalized residual is a valid distribution on $\Delta^{N-1}$ (for
 $p_h < 1$).

\paragraph{Step 2 (pure sink residual).}
For a sink head with $p_{\mathrm{sink}} = 1-\delta$ and uniform
residual, $\Attn{l,h_{\mathrm{sink}}}(t) = \delta/N$ off $\Ssink$, so
 $\tilde{\mathbf{A}}_{l,h_{\mathrm{sink}}}(t) = \frac{\delta/N}{\delta}
= \frac1N$ on all $N$ non-sink positions.

\paragraph{Step 3 (norm).}
 $\|\tilde{\mathbf{A}}_{l,h_{\mathrm{sink}}}\|_2^2 = N \cdot N^{-2} =
\frac1N$.

\paragraph{Step 4 (entropy).}
 $\Hsink = -\ln(1/N) = \ln N = \ln(T - |\Ssink|) \to +\infty$ as
 $T \to \infty$.

\paragraph{Step 5 (sink weight decay).}
With $u_h \triangleq \exp(-\Hsink(A_{l,h})/\tau_g)$:
 $u_{\mathrm{sink}} = \exp(-\ln N/\tau_g) = N^{-1/\tau_g} \to 0$.

\paragraph{Step 6 (retrieval weight floor).}
For a retrieval head with
 $\Attn{l,h^{*}}(t^{*}) = 1-\epsilon$, $t^{*}\notin\Ssink$, and sink
mass $p_{h^{*}} \ge 0$:
 $\tilde{\mathbf{A}}(t^{*}) = \frac{1-\epsilon}{1-p_{h^{*}}} \ge
1-\epsilon$, so $\|\tilde{\mathbf{A}}\|_2^2 \ge (1-\epsilon)^2$ and
 $\Hsink \le 2\ln\frac{1}{1-\epsilon} \triangleq \eta_{\mathrm{ret}} <
\infty$, giving $u_{h^{*}} \ge C_{\mathrm{ret}} > 0$ for all $T$.

\paragraph{Step 7 (normalized suppression).}
 $w_{l,h_{\mathrm{sink}}} = \frac{u_{\mathrm{sink}}}{u_{h^{*}} +
\sum_{j\neq h^{*}} u_j} \le \frac{u_{\mathrm{sink}}}{C_{\mathrm{ret}} +
u_{\mathrm{sink}}}$ since all kernels are non-negative. Taking
 $T \to \infty$ with $u_{\mathrm{sink}} \to 0$: the sink head's
normalized weight vanishes.

\subsection{Budget initialization and exact conservation}
\label{app:budget}

\paragraph{Step 1 (dimensional decomposition).}
In PagedAttention, memory is allocated as physical blocks $P_b$ of
capacity $B$. Under GQA, the $\HQ$ query heads partition into $\HKV$ groups sharing per-group quotas $K_{l,g} = K_l$. Total capacity:
 $\mathcal{C}_{\mathrm{total}} = \sum_{l=1}^{L}\HKV\cdot B\cdot K_l$.
The quota divisor is $\HKV\cdot B$ (not $r\cdot B$ or $\HQ\cdot B$):
 $\mathcal{B}_l$ is the layer's total token budget across all $\HKV$ buffers, so each buffer receives $\mathcal{B}_l/\HKV$, and each page
stores $B$ tokens.

\paragraph{Step 2 (fractional residuals).}
The unquantized page target per group is
 $\kappa_l \triangleq \frac{\mathcal{B}_l}{\HKV B}$, decomposed as
 $\kappa_l = K_l + R_l$ with $K_l = \lfloor\kappa_l\rfloor$ and
 $R_l = \{\kappa_l\} \in [0,1)$.

\paragraph{Step 3 (unadjusted conservation condition).}
 $\mathcal{C}_{\mathrm{unadj}} = \HKV B \sum_{l=1}^{L} K_l
= \mathcal{B}_{\mathrm{total}} - \HKV B\sum_{l=1}^{L} R_l$. Exact
conservation iff $\sum_l R_l = 0$, which (each $R_l \ge 0$) holds iff
 $R_l = 0$ for all $l$, i.e.,
 $\mathcal{B}_{\mathrm{total}} \equiv 0 \pmod{L\cdot\HKV\cdot B}$.

\paragraph{Step 4 (deficit integrality).}
Define $m \triangleq \sum_{l=1}^{L} R_l =
\frac{\mathcal{B}_{\mathrm{total}}}{\HKV B} - \sum_l K_l$. Under
page-vector budgets, both terms are integers, so $m \in \mathbb{Z}$;
and $m > 0$ with $m < L$, giving $m \in \{1,\dots,L-1\}$.

\paragraph{Step 5 (largest-remainder restoration).}
Sort layers by residual magnitude, $R_{\pi(1)} \ge \cdots \ge
R_{\pi(L)}$ (ties broken by canonical layer index), and set
 $K^{\mathrm{adjusted}}_{\pi(i)} = K_{\pi(i)} + 1$ for $1 \le i \le m$,
 $K_{\pi(i)}$ otherwise. Then $\sum_l K^{\mathrm{adjusted}}_l = \sum_l
K_l + m$, so $\mathcal{C}_{\mathrm{adjusted}} = \HKV B(\sum_l K_l + m)
= \mathcal{B}_{\mathrm{total}}$, with zero internal fragmentation.

\paragraph{Numerical cases (Llama-3.1-8B geometry: $L{=}32$,
 $\HKV{=}8$, $r{=}4$, $B{=}16$; granularity $\HKV B = 128$; full-model
 granularity $4{,}096$).}
\emph{Case A (divisible, $\mathcal{B}_{\mathrm{total}} = 32{,}768$):}
 $32{,}768 \bmod 4{,}096 = 0$; $\mathcal{B}_l = 1{,}024$;
 $K_l = 8$ pages, $R_l = 0$; conservation exact.
\emph{Case B (non-divisible, target $35{,}000$):} the nearest
hardware-aligned budget is $\lfloor 35{,}000/128\rfloor \times 128 =
273\times 128 = 34{,}944$; $\kappa_l = 8.53125$; $K_l = 8$,
 $R_l = 0.53125$ for all $l$; $m = 17$ pages. Apportionment: layers
 $\{1,\dots,17\}$ receive $9$ pages, layers $\{18,\dots,32\}$ receive
 $8$; verification: $17\times 9\times 128 + 15\times 8\times 128 =
34{,}944$---exact conservation, zero fragmentation.

\subsection{Engine-level page accounting}
\label{app:pageaccounting}

\paragraph{Step 1 (exact per-group cardinality).}
By \cref{eq:page_aggregation} with deterministic tie-breaking,
 $|\RetSet{l}{g}| \equiv K_l$ for every group.

\paragraph{Step 2 (lower bound).}
Each $\RetSet{l}{g} \subseteq P^{\mathrm{ret}}_l$, so by monotonicity
 $K_l \le |P^{\mathrm{ret}}_l|$.

\paragraph{Step 3 (lower-bound equality).}
If all groups select the identical set
 $\mathcal{K}^{*}$, the union is $\mathcal{K}^{*}$ and
 $|P^{\mathrm{ret}}_l| = K_l$. Conversely, if
 $|P^{\mathrm{ret}}_l| = K_l$ with each
 $\RetSet{l}{g} \subseteq P^{\mathrm{ret}}_l$ of the same cardinality,
each $\RetSet{l}{g} = P^{\mathrm{ret}}_l$: all selections are
identical.

\paragraph{Step 4 (upper bound).}
By Boole's inequality,
 $|P^{\mathrm{ret}}_l| \le \sum_g |\RetSet{l}{g}| = \HKV K_l$.

\paragraph{Step 5 (upper-bound equality).}
If the selections are pairwise disjoint, finite additivity gives
 $|P^{\mathrm{ret}}_l| = \HKV K_l$. Conversely, if some page lies in
two selections, it is counted at least twice in the sum but exactly
once in the union, so equality forces pairwise disjointness.

\paragraph{Step 6 (page-union ratio).}
 $\mathrm{PUR}_l \triangleq |P^{\mathrm{ret}}_l|/K_l$ satisfies
 $1.0 \le \mathrm{PUR}_l \le \HKV$: it tends to $1$ when groups attend
to shared regions (system prompts, preambles, shared code) and to
 $\HKV$ when groups specialize into independent subspaces.

\subsection{Strict physical memory bounds}
\label{app:memory}

\paragraph{Step 1 (union semantics).}
All $r$ sibling heads read the same physical slice
 $\mathbf{K}_{\mathrm{physical}}[:,g,:,:]$, so a shared buffer cannot
store token $t$ for one head while freeing it for another: the engine
materializes
 $\mathcal{K}^{\mathrm{union}}_{g} = \bigcup_{h\in\Group{g}}\TokSet{h}$ and $M_{\mathrm{indep}} = |\mathcal{K}^{\mathrm{union}}_{g}|$.

\paragraph{Steps 2--3 (bounds under head-independent selection).}
Subset inclusion plus cardinality monotonicity gives $k \le
M_{\mathrm{indep}}$, with equality iff all heads select identical
sets. Subadditivity gives $M_{\mathrm{indep}} \le rk$, and
 $\mathcal{K}^{\mathrm{union}}_{g} \subseteq \{1,\dots,T\}$ gives
 $M_{\mathrm{indep}} \le T$; hence $M_{\mathrm{indep}} \le \min(rk, T)$,
with equality iff selections are pairwise disjoint and $rk \le T$.
Dividing by $k$: $1.0 \le \UOR(g,k) \le r$.

\paragraph{Steps 4--5 (pooling preserves the simplex).}
 $\Attn{l,h}(t) > 0$ (softmax outputs) and $w_{l,h} > 0$ (ratios of
positive exponentials), so $\AttnPool{l}{g}(t) \ge 0$. For
normalization, interchange the finite summations and factor the
weights:
\[
  \sum_{t=1}^{T}\AttnPool{l}{g}(t)
  = \sum_{h\in\Group{g}} w_{l,h}
    \Bigl(\sum_{t=1}^{T}\Attn{l,h}(t)\Bigr)
  = \sum_{h\in\Group{g}} w_{l,h}\cdot 1 = 1 .
\]
Hence $\AttnPool{l}{g} \in \Delta^{T-1}$: a convex combination of
probability vectors is a probability vector.

\paragraph{Step 6 (exact cardinality and page extension).}
Sorting $\AttnPool{l}{g}$ via a permutation $\sigma$ of
 $\{1,\dots,T\}$ in non-increasing order, the selection set
 $\{\sigma(1),\dots,\sigma(k)\}$ contains exactly $k$ pairwise distinct
indices because $\sigma$ is a bijection (ties broken by canonical
index): $|\RetSet{l,g}| \equiv k$ for every attention configuration,
temperature, and group ratio, so $\UOR(g,k) = k/k \equiv 1.0$. The
token-level guarantee extends to pages verbatim: block-max reduces the
pooled vector to scalar page scores, and page-level TopK retains
exactly $K_l$ frames per group.

\subsection{Needle retrieval preservation}
\label{app:needle}

\paragraph{Step 1 (residual formulation).}
For each head, $\Attn{l,h}(t) = (1-p_h)\tilde{\mathbf{A}}_{l,h}(t)$ for non-sink $t$, and
 $\|\tilde{\mathbf{A}}_{l,h}\|_2^2 = \exp(-\CollisionSink{\Attn{l}{h}})$.

\paragraph{Step 2 (sibling coordinates vanish, part i).}
From \cref{eq:sibling},
 $\|\tilde{\mathbf{A}}_{l,j}\|_2^2 \le e^{c_j}/(T-|\Ssink|)$, so
 $\tilde{\mathbf{A}}_{l,j}(t^{*}) \le e^{c_j/2}/\sqrt{T-|\Ssink|}$ and,
since $(1-p_j) \le 1$,
 $\Attn{l,j}(t^{*}) \le e^{c_j/2}/\sqrt{T-|\Ssink|} \to 0$; by the
squeeze theorem the limit is zero for all $j \neq h^{*}$. Mean pooling
therefore satisfies $\AttnMean{l}{g}(t^{*}) \to \frac{1-\epsilon}{r}$.

\paragraph{Step 3 (retrieval kernel floor).}
 $\CollisionSink{\Attn{l}{h^{*}}} \le \eta_{\mathrm{ret}}$ gives, after
multiplying by $-1/\tau_g < 0$ (which reverses the inequality) and
exponentiating, $u_{h^{*}} \ge C_{\mathrm{ret}} > 0$ for all $T$.

\paragraph{Step 4 (sibling kernel sum vanishes).}
Each $u_j \le e^{c_j/\tau_g}(T-|\Ssink|)^{-1/\tau_g}$; defining
 $D(T) \triangleq \bigl(\sum_{j\neq h^{*}} e^{c_j/\tau_g}\bigr)
(T-|\Ssink|)^{-1/\tau_g}$, the coefficient is $T$-independent and the
power term vanishes, so $\sum_{j\neq h^{*}} u_j \le D(T) \to 0$.

\paragraph{Step 5 (weights converge).}
 $w_{l,h^{*}} = \frac{1}{1 + \frac{\sum_{j\neq h^{*}}u_j}{u_{h^{*}}}}$ after dividing by $u_{h^{*}} > 0$; the ratio is sandwiched between
zero and $\frac{D(T)}{C_{\mathrm{ret}}} \to 0$, so
 $w_{l,h^{*}} \to 1$ and $\sum_{j\neq h^{*}} w_{l,j} \to 0$.

\paragraph{Step 6 (squeeze on the needle).}
 $\AttnPool{l}{g}(t^{*}) = w_{l,h^{*}}(1-\epsilon) +
\sum_{j\neq h^{*}} w_{l,j}\Attn{l,j}(t^{*})$ with the sibling term
sandwiched between zero and $\sum_{j\neq h^{*}} w_{l,j} \to 0$: both
brackets converge to $1-\epsilon$, giving the limit
 $\lim_{T\to\infty}\AttnPool{l}{g}(t^{*}) = 1-\epsilon$.

\subsection{Finite-context retention bound}
\label{app:finite}

\paragraph{Step 1 (weight floor).}
 $g(x,y) \triangleq x/(x+y)$ satisfies
 $\partial g/\partial x = \frac{y}{(x+y)^2} \ge 0$ and
 $\partial g/\partial y = -\frac{x}{(x+y)^2} < 0$ for $x>0, y\ge0$;
replacing $x = u_{h^{*}}$ by its floor $C_{\mathrm{ret}}$ and
 $y = \sum_{j\neq h^{*}} u_j$ by its ceiling $D(T)$ preserves the lower
bound: $w_{l,h^{*}} \ge \frac{C_{\mathrm{ret}}}{C_{\mathrm{ret}} +
D(T)}$.

\paragraph{Step 2 (minorization).}
The sibling term $\sum_{j\neq h^{*}} w_{l,j}\Attn{l,j}(t^{*}) \ge 0$;
dropping it gives
 $\AttnPool{l}{g}(t^{*}) \ge w_{l,h^{*}}(1-\epsilon) \ge
\frac{C_{\mathrm{ret}}}{C_{\mathrm{ret}}+D(T)}(1-\epsilon)$.

\paragraph{Step 3 (deficit closed form).}
Writing the bound as $(1-\epsilon) - \Delta(T)$ and isolating,
 $\Delta(T) = (1-\epsilon)\frac{D(T)}{C_{\mathrm{ret}}+D(T)}$.

\paragraph{Step 4 (monotone decay).}
 $\frac{d}{dT}D(T) = -\frac{1}{\tau_g}
\bigl(\sum_{j\neq h^{*}}e^{c_j/\tau_g}\bigr)
(T-|\Ssink|)^{-(\frac{1}{\tau_g}+1)} < 0$, and by the chain rule
 $\frac{d\Delta}{dT} < 0$: the bound tightens as context grows.

\paragraph{Step 5 (exact instantiation and scaling; arithmetic, not
 measurement).}
With $r=8$ (7 siblings), $|\Ssink|=4$, $\tau_g=0.5$, $c_j=0$,
 $\eta_{\mathrm{ret}}=\ln 2$, and $1-\epsilon=0.95$:
 $C_{\mathrm{ret}} = 2^{-2} = 0.25$;
 $D(T) = 7(T-4)^{-2}$;
 $\Delta(T) = 0.95\cdot\frac{D(T)}{0.25+D(T)} \le 0.95\cdot 4\,D(T)$.
Because $4.08\times 10^{-10} \ll 0.25$ at $T = 131{,}072$,
 $\Delta(131{,}072) \le 1.63\times 10^{-9}$. In FP32,
 $\varepsilon_{\mathrm{FP32}} \approx 1.19\times 10^{-7}$; since
 $\Delta(T) \ll \varepsilon_{\mathrm{FP32}}$, the pooled needle signal
is numerically indistinguishable from exact single-head selection in
FP32 registers. \Cref{tab:finite_scaling} tabulates the closed-form
bound across deployment contexts.

\paragraph{Two-regime mechanism check.}
\Cref{fig:cor2} exercises the corollary's mechanism end-to-end in the
two regimes the pilot spans. Panel~(a) is a synthetic construction
that satisfies the premises of \cref{thm:needle_preservation} exactly
(a needle head at mass $1-\epsilon = 0.95$ with near-maximal-entropy
siblings), run through the full engine path: the measured pooled
needle mass coincides with the bound $(1-\epsilon)-\Delta(T)$ to
within $3\times 10^{-7}$ at every context length, up to and including
128k. Panel~(b) is the real-text pilot regime on
Qwen2.5-1.5B-Instruct eager-attention captures at $T \le 2048$: the
measured best-head needle mass is $\approx 0.16$ rather than $0.95$,
the premise is unsatisfied, and entropy-weighted pooling degenerates
toward arithmetic-mean pooling
(\cref{prop:temperature_limits}~ii), with the two pooled masses within
a factor $1.06$ of each other. Both panels appear together
deliberately: panel~(a) alone could be misread as an empirical claim,
whereas the pair delimits precisely where the guarantee binds and
where it does not yet.

\begin{figure*}[t]
\centering
\includegraphics[width=0.95\textwidth]{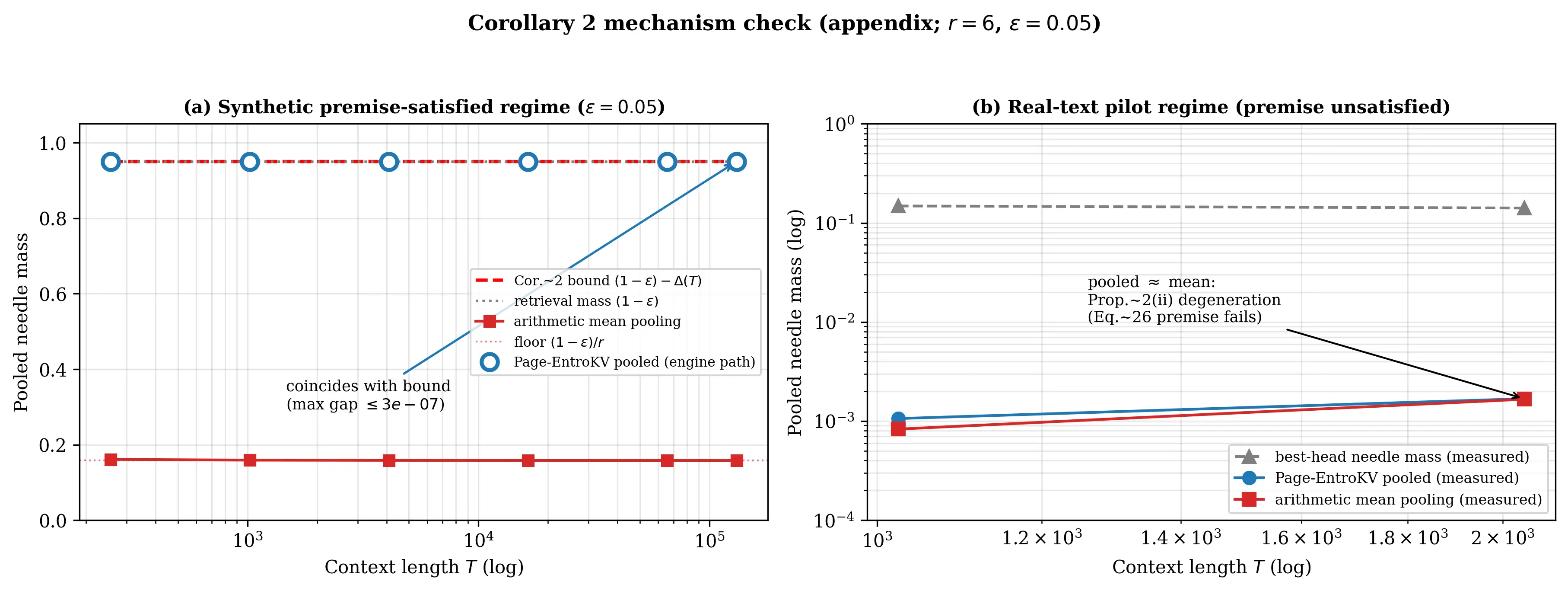}
\caption{Corollary~\ref{cor:finite_context_bound} two-regime mechanism
  check ($r{=}6$, single $\varepsilon = 0.05$ used throughout,
  $\tau_g = 0.5$). (a)~\emph{Synthetic premise-satisfied regime}: an
  engine-path construction with a needle head at mass $1-\varepsilon$   and near-maximal-entropy siblings; the measured pooled needle mass
  coincides with the bound $(1-\varepsilon)-\Delta(T)$ at every context
  length (raw gap $\le 3\times 10^{-7}$; markers offset by $+0.01$   vertically for visibility), verifying the mechanism end-to-end up to
  128k. (b)~\emph{Real-text pilot regime (premise unsatisfied)}:
  Qwen2.5-1.5B-Instruct eager-attention captures at $T \le 2048$; the
  measured best-head needle mass is $\approx 0.16$ rather than
  $0.95$, so the premise of \cref{thm:needle_preservation} is not
  reached and entropy-weighted pooling degenerates toward
  arithmetic-mean pooling (\cref{prop:temperature_limits}~ii), the
  pooled and mean masses agreeing to within a factor $1.06$. The two
  panels together delimit where the guarantee binds; neither panel
  alone is interpretable.}
\label{fig:cor2}
\end{figure*}

\begin{table*}[t]
\centering\small
\caption{Non-asymptotic scaling of the finite-context bound
  (\cref{cor:finite_context_bound}) across deployment contexts
  ($r{=}8$, $\tau_g{=}0.5$, $c_j{=}0$, $\eta_{\mathrm{ret}}=\ln2$,
  $1-\epsilon = 0.95$). All entries are arithmetic instantiations of
  the closed-form bound---not measurements.}
\label{tab:finite_scaling}
\begin{tabular}{@{}lrrrrrr@{}}
\toprule
 $T$ & Non-sink $N$ & $D(T)$ & $\Delta(T)$ &
 $\AttnPool{l}{g}(t^{*})$ & $\AttnMean{l}{g}(t^{*})$ & Ratio \\
\midrule
4k    & 4{,}092   & $4.18\times 10^{-7}$  & $1.59\times 10^{-6}$ & 0.949998 & 0.118750 & $8.00\times$ \\
8k    & 8{,}188   & $1.04\times 10^{-7}$  & $3.97\times 10^{-7}$ & 0.949999 & 0.118750 & $8.00\times$ \\
16k   & 16{,}380  & $2.61\times 10^{-8}$  & $9.92\times 10^{-8}$ & 0.949999 & 0.118750 & $8.00\times$ \\
32k   & 32{,}764  & $6.52\times 10^{-9}$  & $2.48\times 10^{-8}$ & 0.950000 & 0.118750 & $8.00\times$ \\
64k   & 65{,}532  & $1.63\times 10^{-9}$  & $6.20\times 10^{-9}$ & 0.950000 & 0.118750 & $8.00\times$ \\
128k  & 131{,}068 & $4.08\times 10^{-10}$ & $1.63\times 10^{-9}$ & 0.950000 & 0.118750 & $8.00\times$ \\
\bottomrule
\end{tabular}
\end{table*}

\subsection{Page-frame monotonicity}
\label{app:monotone}

\paragraph{Steps 1--3 (exact equality at the target page).}
The pages $\{P_b\}$ form a disjoint partition of $\{1,\dots,T\}$, so
 $t^{*}$ lies in a unique $P^{*}$. By definition of the maximum,
 $I_{l,g}(P^{*}) \ge \AttnPool{l}{g}(t^{*})$; conversely
 $P^{*} \subseteq \{1,\dots,T\}$ gives
 $\max_{t\in P^{*}}\AttnPool{l}{g}(t) \le \max_{t}\AttnPool{l}{g}(t) =
\AttnPool{l}{g}(t^{*})$. The squeeze gives the identity
 $I_{l,g}(P^{*}) \equiv \AttnPool{l}{g}(t^{*})$: block-max preserves
the peak coordinate exactly, independent of $B$.

\paragraph{Step 4 (dominance over other pages).}
For any $P_j \neq P^{*}$ (disjoint from $P^{*}$, so
 $t^{*}\notin P_j$), let $t_j$ be the local maximizer. Then
 $I_{l,g}(P_j) = \AttnPool{l}{g}(t_j) \le \AttnPool{l}{g}(t^{*}) =
I_{l,g}(P^{*})$, with strict inequality when $t^{*}$ is the unique
global maximizer.

\paragraph{Step 5 (retention for any quota).}
Sorting page scores in non-increasing order, the maximal page occupies
the first rank; for any budget $K_l \ge 1$,
 $P^{*} \in \RetSet{l}{g}$: the page containing the global attention
peak is unconditionally retained.

\subsection{On-chip execution and the FP16 guard}
\label{app:fp16}

\paragraph{Argument bounds.}
For $A \in \Delta^{T-1}$: the collision argument $\|A\|_2^2$ is at
most $1$ (single-coordinate concentration) and at least $1/T > 0$ (uniform), so $\Renyi{2}{A} \in [0, \ln T]$ and the logarithm never
sees a zero argument.

\paragraph{Guard placement.}
The guard must sit inside the logarithm
($-\ln(\|A\|_2^2 + \varepsilon)$, $\varepsilon = 10^{-12}$) and never
inside the power: expanding
 $\sum_t(A(t)+\varepsilon)^2 = \|A\|_2^2 + 2\varepsilon +
T\varepsilon^2$, the context-dependent term $T\varepsilon^2$ corrupts
the normalization $\sum_t A(t) = 1$ and breaks the convexity that
\cref{thm:memory_guarantee} relies on. Guarding inside the logarithm
leaves the distribution untouched.

\paragraph{Degenerate all-sink clamping.}
If $\sum_{t'\in\Ssink}\Attn{l,h}(t') = 1$, the residual denominator is
zero and $\tilde{\mathbf{A}}$ is undefined. The implementation clamps
to \texttt{max(1-p\_h, 1e-12)}: the residual coordinates stay $0$, so
 $\|\tilde A\|_2^2 = 0$ and $\Hsink = -\ln(10^{-12}) \approx 27.63$%
---maximal entropy. Under $\tau_g = 0.5$,
 $u_{\mathrm{sink}} = \exp(-55.26) \approx 9.99\times 10^{-25}$: the
degenerate head's pooling weight vanishes, preventing all-sink heads
from diluting genuine retrieval signals.

\section{Reference implementation}
\label{app:implementation}

\Cref{lst:page_entrokv} gives the PyTorch reference implementation of
the \method{} engine: sink isolation (\cref{def:sinkentropy}), discrete
R\'enyi entropy with the $\varepsilon$-guard applied inside the
logarithm only (preserving normalization; cf.\ \cref{rem:fp16}),
Boltzmann weighting over sibling heads (\cref{eq:boltzmann_weights}),
group pooling (\cref{eq:pooled_attention}), block-max page aggregation
and top-$K$ physical block selection (\cref{eq:page_aggregation}). The
listing evaluates entropies on the full non-sink sequence for
clarity---the full-sequence instance of the observation-window
restriction (\cref{def:obswindow}); the windowed instance follows by
restricting the inputs to $W_{\mathrm{obs}}$.

\begin{lstlisting}[
  language=Python,
  caption={PyTorch reference implementation of the \method{}
    engine, with sink-isolated entropy weighting.},
  label={lst:page_entrokv}
]
import torch
import torch.nn as nn
import torch.nn.functional as F


class PageEntroKVGroupPooler(nn.Module):
    """Entropy-weighted GQA group pooling + page mapping.

    Entropy is computed on the sink-isolated distribution; the pooled
    output uses the RAW attention rows, so it remains a convex
    combination of probability vectors (the budget-preservation
    guarantee relies on this).
    """

    def __init__(self, num_query_heads=32, num_kv_heads=8,
                 block_size=16, tau_g=0.5, alpha=2.0,
                 sink_size=4):
        super().__init__()
        self.num_query_heads = num_query_heads
        self.num_kv_heads = num_kv_heads
        self.group_size = num_query_heads // num_kv_heads
        self.block_size = block_size
        self.tau_g = tau_g
        self.alpha = alpha
        self.sink_size = sink_size

    def renyi_entropy(self, attn):
        # eps inside log ONLY: preserves sum(p) = 1
        eps = 1e-12
        if self.alpha == 1.0:
            return -torch.sum(
                attn * torch.log(attn + eps), dim=-1)
        s = torch.sum(torch.pow(attn, self.alpha), dim=-1)
        return (1.0 / (1.0 - self.alpha)) * torch.log(s + eps)

    @torch.inference_mode()
    def forward(self, attn_weights, num_retained_blocks):
        """[B, H_Q, T] -> bool mask [B, H_KV, num_blocks]."""
        Bsz, H_Q, T = attn_weights.shape
        sink_mask = torch.zeros(
            T, dtype=torch.bool, device=attn_weights.device)
        sink_mask[:min(self.sink_size, T)] = True
        attn_nonsink = attn_weights.masked_fill(
            sink_mask.unsqueeze(0).unsqueeze(0), 0.0)
        attn_nonsink = attn_nonsink / attn_nonsink.sum(
            dim=-1, keepdim=True).clamp(min=1e-12)
        ent = self.renyi_entropy(attn_nonsink)
        grouped = attn_weights.view(
            Bsz, self.num_kv_heads, self.group_size, T)
        ent_g = ent.view(Bsz, self.num_kv_heads,
                         self.group_size)
        weights = F.softmax(-ent_g / self.tau_g,
                            dim=-1).unsqueeze(-1)
        pooled = (weights * grouped).sum(dim=2)
        pad = (self.block_size - T % self.block_size) \
              % self.block_size
        if pad:
            pooled = F.pad(pooled, (0, pad))
        nb = pooled.shape[-1] // self.block_size
        pages = pooled.view(
            Bsz, self.num_kv_heads, nb,
            self.block_size).max(dim=-1).values
        idx = torch.topk(pages, num_retained_blocks,
                         dim=-1).indices
        mask = torch.zeros_like(pages, dtype=torch.bool)
        mask.scatter_(-1, idx, True)
        return mask
\end{lstlisting}

\end{document}